%% file: adv-omd.tex
\documentclass[oneside,11pt]{article} 

\renewenvironment{abstract}
  {{\centering\large\bfseries Abstract\par}\vspace{0.7ex}%
    \bgroup
       \leftskip 20pt\rightskip 20pt\small\noindent\ignorespaces}%
  {\par\egroup\vskip 0.25ex}

{\par\egroup\vskip 0.25ex}

\usepackage{amsmath,amsfonts,amssymb,amsthm,commath}

\usepackage{algorithm,algorithmic}

\usepackage[utf8]{inputenc} 
\usepackage[T1]{fontenc}    
\usepackage{microtype}   

\usepackage{nicefrac}       

\usepackage{booktabs,enumitem}

\usepackage{graphicx}

\usepackage{url}
\usepackage{xcolor}

\usepackage{hyperref} 

\usepackage{balance} 

\usepackage{pifont}

\newcommand{\cmark}{\text{\ding{51}}}
\newcommand{\xmark}{\text{\ding{55}}}

\numberwithin{equation}{section}

\newcommand{\calB}{\mathcal{B}}
\newcommand{\calC}{\mathcal{C}}
\newcommand{\calD}{\mathcal{D}}

\newcommand{\calF}{\mathcal{F}}

\newcommand{\calH}{\mathcal{H}}

\newcommand{\calK}{\mathcal{K}}

\newcommand{\calX}{\mathcal{X}}

\newcommand{\R}{\mathbb{R}}

\newcommand{\Rd}{\mathbb{R}^d}

\newcommand{\zeronorm}[1]{\left\lVert #1 \right\rVert_{0}}
\newcommand{\twonorm}[1]{\left\lVert #1 \right\rVert}
\newcommand{\onenorm}[1]{\left\lVert #1 \right\rVert_{1}}
\renewcommand{\abs}[1]{\left\lvert #1 \right\rvert}

\newcommand{\infnorm}[1]{\left\lVert #1 \right\rVert_{\infty}}
\newcommand{\inner}[2]{\left\langle #1, #2 \right\rangle}

\DeclareMathOperator*{\argmin}{arg\,min}

\newcommand{\EXP}{\operatorname{\mathbb{E}}} 

\renewcommand{\Pr}{\operatorname{Pr}}

\newcommand{\poly}{\operatorname{poly}}
\newcommand{\polylog}{\operatorname{polylog}}
\newcommand{\sign}{\operatorname{sign}}

\newcommand{\ind}[1]{\boldsymbol{1}_{\{#1\}}}

\theoremstyle{plain}

\ifx\theorem\undefined
\newtheorem{theorem}{Theorem}
\fi

\ifx\lemma\undefined
\newtheorem{lemma}[theorem]{Lemma}
\fi

\ifx\proposition\undefined
\newtheorem{proposition}[theorem]{Proposition}
\fi

\ifx\corollary\undefined
\newtheorem{corollary}[theorem]{Corollary}
\fi

\theoremstyle{definition}

\ifx\remark\undefined
\newtheorem{remark}[theorem]{Remark}
\fi

\ifx\definition\undefined
\newtheorem{definition}[theorem]{Definition}
\fi

\ifx\conjecture\undefined
\newtheorem{conjecture}[theorem]{Conjecture}
\fi

\ifx\fact\undefined
\newtheorem{fact}[theorem]{Fact}
\fi

\ifx\claim\undefined
\newtheorem{claim}[theorem]{Claim}
\fi

\ifx\assumption\undefined
\newtheorem{assumption}{Assumption}
\fi

\usepackage{xspace}

\newcommand{\refine}{\textsc{Refine}\xspace}
\newcommand{\init}{\textsc{Initialize}\xspace}

\newcommand{\oraclexy}{\mathrm{EX}}

\newcommand{\oracley}{\mathrm{EX}_y}

\newcommand{\err}{\mathrm{err}}

\newcommand{\haty}{\hat{y}}
\newcommand{\hatw}{\hat{w}}
\newcommand{\breg}{\calB_{\Phi}}
\newcommand{\optu}{\tilde{u}}

\newcommand{\wavg}{w_{\mathrm{avg}}}
\newcommand{\wsharp}{w^{\sharp}}

\newcommand{\OPT}{\mathrm{OPT}}

\newcommand{\pnorm}[1]{\norm{#1}_p}
\newcommand{\qnorm}[1]{\norm{#1}_q}

\newcommand{\citet}{\cite}
\newcommand{\citep}{\cite}

\title{On the Power of Localized Perceptron for Label-Optimal Learning of Halfspaces with Adversarial Noise}
\author{
{\bfseries Jie Shen}\\
Stevens Institute of Technology\\
\texttt{jie.shen@stevens.edu}
}

\begin{document}
\maketitle

\input{intro.tex}

\input{setup.tex}

\input{algorithm.tex}

\input{guarantee.tex}

\input{conclusion.tex}

\clearpage
\bibliographystyle{alpha}
\bibliography{../../jshen_ref}

\clearpage

\appendix

\input{appendix.tex}

\end{document}

%% file: intro.tex
\begin{abstract}
We study {\em online} active learning of homogeneous halfspaces in $\mathbb{R}^d$ with adversarial noise where the overall probability of a noisy label is constrained to be at most $\nu$. Our main contribution is a Perceptron-like online active learning algorithm that runs in polynomial time, and under the conditions that the marginal distribution is isotropic log-concave and $\nu = \Omega(\epsilon)$, where $\epsilon \in (0, 1)$ is the target error rate, our algorithm PAC learns the underlying halfspace with near-optimal label complexity of $\tilde{O}\big(d \cdot \polylog(\frac{1}{\epsilon})\big)$ and sample complexity of $\tilde{O}\big(\frac{d}{\epsilon} \big)$.\footnote{We use the notation $\tilde{O}(f) := O(f \cdot \log f)$, $\tilde{\Omega}(f) := \Omega(f / \log f)$, and $\tilde{\Theta}(f)$ that is between $\tilde{\Omega}(f)$ and $\tilde{O}(f)$.} Prior to this work, existing online algorithms designed for tolerating the adversarial noise are  subject to either label complexity polynomial in $\frac{1}{\epsilon}$, or suboptimal noise tolerance, or restrictive marginal distributions. With the additional prior knowledge that the underlying halfspace is $s$-sparse, we obtain attribute-efficient label complexity of $\tilde{O}\big( s \cdot \polylog(d, \frac{1}{\epsilon}) \big)$ and sample complexity of $\tilde{O}\big(\frac{s}{\epsilon} \cdot \polylog(d) \big)$. As an immediate corollary, we show that under the agnostic model where no assumption is made on the noise rate $\nu$, our active learner achieves an error rate of $O(\OPT) + \epsilon$ with the same running time and label and sample complexity, where $\OPT$ is the best possible error rate achievable by any homogeneous halfspace. 
\end{abstract}

\section{Introduction}

In many practical applications, there are massive amounts of unlabeled data but labeling is expensive. This distinction has driven the study of active learning~\cite{cohn1994improving,balcan2007margin,dasgupta2011active,hanneke2014theory,awasthi2017power}, where labels are initially hidden and the learner must pay for each label it wishes to be revealed. The goal is to design querying strategies to avoid less informative labeling requests, e.g. the labels that can be inferred from previously seen samples. Parallel to active learning, online learning concerns the scenario where the learner observes a stream of samples and makes real-time model updating in order to compete with the best model obtained by seeing all the history data in a batch~\cite{rosenblatt1958perceptron,littlestone1989weighted,cesa1996worst,zinkevich2003online,shalev2012online,hazan2019introduction}. Online learning algorithms have also been broadly investigated in machine learning, and have found various successful applications owing to its potential of savings in memory cost, low computational cost per sample, and its generalization ability~\cite{cesa2004generalization,kakade2008generalization}.

In this paper, we study the important problem of active learning of homogeneous halfspaces in the online setting, where the learner observes a stream of unlabeled data and makes spot decision of whether or not to query the labels. The goal is to achieve the best of the two worlds: label efficiency from active learning and computational efficiency from online learning. In this spectrum, there are a number of early works that share the same merit with this paper. For example, \citet{freund1997selective} proposed a query-by-committee learning algorithm and \citet{dasgupta2005analysis} developed a Perceptron-like algorithm, both of which are implemented in an online fashion and enjoy a near-optimal label complexity bound of $\tilde{O}(d \log\frac{1}{\epsilon})$ where $d$ is the dimension of the instance\footnote{We will interchangeably use ``instance'' and ``unlabeled data''.} and $\epsilon \in (0, 1)$ is the target error rate. However, there are two crucial assumptions made by these works that seem too stringent:~1)~the marginal distribution over the unlabeled data is uniform on the unit sphere in $\Rd$; and 2)~there exists a perfect halfspace that incurs zero error rate with respect to the underlying distribution. In this regard, a natural question is: {\em can we design an online active learner, that provably works under a significantly more general family of marginal distributions while achieving arbitrarily small error rate without the realizability condition?} 

To be more concrete, we are interested in designing an online active learning algorithm that PAC learns some underlying halfspace~\cite{valiant1984theory} when the instances are drawn from an isotropic log-concave distribution~\cite{lovasz2007geometry} and the labels are corrupted by the adversarial noise~\cite{haussler1992decision,kearns1992toward}. It is worth mentioning that the family of isotropic log-concave distributions is a significant generalization of the uniform distribution since it includes a variety of prevalent distributions such as Gaussian, exponential, logistic. It is also known that establishing performance guarantees under this family is often technically subtle compared to that of uniform distribution due to the asymmetricity nature of log-concave distributions~\cite{vempala2010random,balcan2013active,diakonikolas2018learning}. Returning to the noise model, we note that the adversarial noise, where the adversary may choose an {\em arbitrary} joint distribution such that the unlabeled data distribution is isotropic log-concave and the overall probability of a noisy label is constrained to be at most $\nu$, is a realistic yet remarkably challenging regime, as suggested by many hardness results~\cite{feldman2006new,guruswami2009hardness,daniely2016complexity,diakonikolas2020near,balcan2020noise}.

Under different assumptions on the adversarial noise rate (which might be suboptimal), a large body of works have established PAC guarantees for online learning with adversarial noise. Unfortunately, none of them resolves the aforementioned question in full. For example, under uniform marginal distributions, Theorem~3 of \citet{kalai2005agnostic} showed that a surprisingly simple averaging scheme already is able to tolerate noise rate $\nu = \tilde{\Omega}(\epsilon)$, though the label and sample complexity are both $\tilde{O}(d^2/\epsilon^2)$.\footnote{\citet{kalai2005agnostic} analyzed two algorithms: a (batch) polynomial regression and an online averaging. Here we are referring to the online averaging approach. We defer the comparison with the former to Section~\ref{subsec:agnostic} when we are in the position to discuss the connection between adversarial noise and agnostic learning.} The same algorithm was then revisited by \citet{klivans2009learning} under isotropic log-concave distributions, with a worse noise tolerance of $\nu = \tilde{\Omega}(\epsilon^3)$. Very recently, \citet{diakonikolas2020non} proposed a novel objective function by optimizing which with projected online gradient descent, one is guaranteed to tolerate the adversarial noise of $\nu = \tilde{\Omega}(\epsilon)$. Notably, their analysis applies to marginal distributions that are more general than isotropic log-concave. Compared to these passive learning algorithms which have label and sample complexity polynomial in $\frac{1}{\epsilon}$, the work of \citet{yan2017revisiting} is more in line with this paper in the sense that they considered the active learning setting. Hence, their label complexity bound has an exponentially better dependence on $\frac{1}{\epsilon}$. However, the noise tolerance of \citet{yan2017revisiting} reads as $\nu = \tilde{\Omega}(\epsilon/\log d)$ and their analysis is applicable only to uniform distributions due to the crucial need of the symmetricity of marginal distributions. As we can see in Table~\ref{tb:comp}, none of the prior works subsumes others.




\begin{table*}[t]
\centering
\caption{Comparison to state-of-the-art online algorithms that are robust to adversarial noise. Prior online algorithms cannot even incorporate attribute efficiency. Even in the non-sparse case (i.e. $s = d$), our algorithm (Theorem~\ref{thm:non-sparse:informal}) has better noise tolerance and works under more general distributions than \citet{yan2017revisiting}, and has improved label and sample complexity compared to \citet{kalai2005agnostic} and \citet{diakonikolas2020non}. 
}
\label{tb:comp}

\vspace{0.1in}

\resizebox{\textwidth}{!}{

\begin{tabular}{lcccc}
\toprule
Work     & Log-concave?  & Label complexity & Sample complexity & Noise tolerance \\
\midrule



Theorem 3 of \citet{kalai2005agnostic} & $\xmark$ & $\tilde{O}(d^2/\epsilon^2)$ &  $\tilde{O}(d^2/\epsilon^2)$ & $\nu = \tilde{\Omega}(\epsilon)$\\

\citet{yan2017revisiting}      & $\xmark$  & $\tilde{O}(d \log\frac{1}{\epsilon})$ & $\tilde{O}(d/\epsilon)$  & $\nu = \tilde{\Omega}(\epsilon/ \log d)$ \\

\citet{diakonikolas2020non}   & $\cmark$  & $\tilde{O}(d/ \epsilon^4)$ & $\tilde{O}(d/ \epsilon^4)$ &$\nu = \tilde{\Omega}(\epsilon)$\\


{\bfseries  This work (Theorem~\ref{thm:non-sparse:informal})}     & $\cmark$   &  $\tilde{O}(d \cdot \polylog(\frac{1}{\epsilon}))$ & $\tilde{O}({d} / {\epsilon})$  &$\nu = \Omega(\epsilon)$\\

{\bfseries  This work (Theorem~\ref{thm:sparse:informal})}     & $\cmark$   &  $\tilde{O}(s \cdot \polylog(d, \frac{1}{\epsilon}))$ & $\tilde{O}(\frac{s}{\epsilon} \cdot \polylog(d))$  &$\nu = \Omega(\epsilon)$\\
\bottomrule
\end{tabular}
}
\end{table*}

\subsection{Main results}

The main contribution of the paper is a novel online active learning algorithm that improves upon the state-of-the-art online algorithms. We  introduce a few useful notations and informally describe our main results in this section; readers are referred to Section~\ref{sec:guarantee} for a precise statement.

Let $\calC$ be the given concept class of halfspaces, $D$ be the joint distribution over $\Rd\times \{-1, 1\}$, and for any $w \in \calC$ define its error rate as $\err_D(w) = \Pr_{(x, y) \sim D}(\sign(w \cdot x) \neq y)$. Our analysis hinges on the following distributional assumptions. 

\begin{assumption}\label{as:x}
The unlabeled data distribution is isotropic log-concave, i.e. it has zero mean and unit covariance matrix, and the logarithm of its density function is concave.
\end{assumption}

\begin{assumption}\label{as:y}
There exists an underlying halfspace $u \in \calC$, such that $\err_D(u) \leq \nu$ for some noise rate $\nu \geq 0$.  
\end{assumption}

Let $\epsilon \in (0, 1)$ be the target error rate given to the learner. We have the following theorem.

\begin{theorem}[Informal]\label{thm:non-sparse:informal}
If Assumptions~\ref{as:x} and \ref{as:y} are satisfied and $\nu \leq O(\epsilon)$, then there is an efficient online active learner that outputs a halfspace $\optu \in \calC$ with $\Pr_{(x, y) \sim D}(\sign(\optu \cdot x) \neq \sign(u \cdot x)) \leq \epsilon$, with label complexity bound of $\tilde{O}\big(d \cdot \polylog( \frac{1}{\epsilon})\big)$ and sample complexity bound of $\tilde{O}\big(\frac{d}{\epsilon}\big)$.
\end{theorem}

We compare with state-of-the-art online algorithms of \citet{kalai2005agnostic,yan2017revisiting,diakonikolas2020non} that are tolerant to adversarial noise. \citet{yan2017revisiting} presented an active learner and obtained analogous label and sample complexity to this work. However, their noise tolerance reads as $\nu = \tilde{\Omega}(\epsilon/\log d)$ while our algorithm is able to tolerate $\nu = \Omega(\epsilon)$; in addition, our results apply to significantly broader marginal distributions. Since \citet{kalai2005agnostic,diakonikolas2020non} considered passive learning, our active learning algorithm naturally enjoys label complexity that has exponentially better dependence on $\epsilon$. Even for the sample complexity, we obtain improved dependence on $d$ and $\epsilon$. On the other side, it is worth mentioning that all these three algorithms run faster than our algorithm. Also, the analysis of \citet{diakonikolas2020non} works under more general marginal distributions, in particular, distributions satisfying concentration, anti-concentration, anti-anti-concentration. Though our results can be generalized to their setting as well (see e.g. \citet{zhang2021improved} for the treatment), we do not pursue it in the paper for the sake of clean presentation.

In addition to the properties aforementioned, we show that our algorithm can essentially incorporate attribute efficiency~\cite{littlestone1987learning}. That is, when the concept class consists of $s$-sparse halfspaces, the obtained label and sample complexity scale as $\tilde{O}(s \cdot \polylog(d))$. This characteristic is especially useful when there is limited availability of samples, a problem that has been studied for decades in machine learning and statistics; see, e.g.~\citet{chen1998atomic,tibshirani1996regression,klivans2004toward,candes2005decoding,feldman2007attribute,plan2017high}.

We have the following result for learning sparse halfspaces.

\begin{theorem}[Informal]\label{thm:sparse:informal}
If Assumptions~\ref{as:x} and \ref{as:y} are satisfied, $\nu \leq O(\epsilon)$, and $u$ is $s$-sparse, then there is an efficient online active learner that outputs an $s$-sparse halfspace $\optu \in \calC$ with $\Pr_{(x, y) \sim D}(\sign(\optu \cdot x) \neq \sign(u \cdot x)) \leq \epsilon$, with label complexity bound of $\tilde{O}\big(s \cdot \polylog(d, \frac{1}{\epsilon})\big)$ and sample complexity bound of $\tilde{O}\big(\frac{s}{\epsilon} \cdot \polylog(d)\big)$.
\end{theorem}

Observe that Theorem~\ref{thm:non-sparse:informal} is a special case of the above by setting $s = d$. We note that the state-of-the-art online algorithms of \citet{kalai2005agnostic,yan2017revisiting,diakonikolas2020non} do not enjoy attribute efficiency~--~it is yet nontrivial for them to encompass this property. Hence we have exponentially better dependence on the dimension $d$ in label and sample complexity. 
See Table~\ref{tb:comp} for a summary of the comparison. 


Finally, through an interesting observation made in \citet{awasthi2017power}, it is possible to translate our main results for the adversarial noise model to the agnostic model of \citet{haussler1992decision,kearns1992toward} where no assumption is made on the noise rate $\nu$. Let $\OPT := \min_{w \in \calC} \err_D(w)$. 

\begin{theorem}[Informal]\label{thm:agnostic:informal}
If Assumption~\ref{as:x} is satisfied, then the algorithm tolerating the adversarial noise outputs a halfspace $\optu$ with $\err_D(\optu) \leq O(\OPT) + \epsilon$, with same label and sample complexity as in Theorem~\ref{thm:sparse:informal}.
\end{theorem}



\subsection{Overview of our techniques}\label{subsec:technique}

Our algorithm is inspired in part by  \citet{zhang2020efficient}. We present an overview of our techniques below, and highlight the algorithmic connection to them as well as the novelty. 

\vspace{0.1in}
\noindent{\bfseries 1) Active learning via stagewise online mirror descent.} The first ingredient in our algorithm is a novel perspective of approaching active learning of halfspaces via stagewise online learning, recently utilized by \citet{zhang2020efficient} for learning halfspaces with benign noise. In each phase, given an initial halfspace $w_0 \in \Rd$, regardless of how the samples are generated, standard regret bound established for online mirror descent with linear loss $\inner{w}{\alpha g_t}$ and $\ell_p$-norm regularizer $\Phi(w)$ implies that the produced sequence of iterates $\{w_{t-1}\}_{t=1}^T$ must satisfy the following with certainty:
\begin{equation}
\frac{1}{T}\sum_{t=1}^{T} \inner{u}{- g_t} \leq \frac{1}{T}\sum_{t=1}^{T} \inner{w_{t-1}}{- g_t}  + \frac{\breg(u; w_0)}{\alpha T} + \frac{\alpha}{T} \sum_{t=1}^T \qnorm{g_t}^2,\label{eq:omd}
\end{equation}
where $u \in \Rd$ is the underlying halfspace we aim to approximate, $\breg(\cdot; \cdot)$ denotes Bregman divergence induced by $\Phi$, and $q \in (0, 1)$ is such that $\frac{1}{p} + \frac{1}{q} = 1$. Our goal is threefold:~attribute efficiency, label efficiency, and small error rate. We will first specify $p \approx 1$ to achieve attribute efficiency which is a well-known technique in online learning~\cite{grove2001general,gentile2003robustness}. In order to reduce the error rate of the initial halfspace $w_0$ with a few label queries, we need to design suitable gradients $g_t$ and choose proper step size $\alpha$ and iteration number $T$ such that a) the right-hand side of \eqref{eq:omd} is as small as  angle $O(\theta(w_0, u))$; and b) the left-hand side is bounded from below by $\theta(\bar{w}, u)$ for certain halfspace $\bar{w}$ that depends on the sequence $\{w_{t-1}\}_{t=1}^T$. It is then possible to show that $\theta(\bar{w}, u) \leq \frac{1}{2} \cdot \theta(w_0, u)$, and we can use $\bar{w}$ as the initial iterate for the next phase of online mirror descent to reduce the distance to $u$ with geometric rate. Therefore, the crucial challenges lie in the design of $g_t$ to accommodate specific noise model and an associated sampling scheme (since $g_t$ depends on the sample). These are also the key technical differences between our work and \citet{zhang2020efficient}, which we elaborate on below.

\vspace{0.1in}
\noindent{\bfseries 2) A semi-random gradient update  for tolerating adversarial noise.} In \citet{zhang2020efficient}, the gradient $g_t$ is heavily tailored to the bounded noise condition~\cite{massart2006risk}. We find it technically hard to reuse it for the adversarial noise since it is well known that the latter is a more involved noise model that will always violate the conditions assumed in the former.\footnote{In the bounded noise model, the adversary is constrained to flip the label of each given instance with probability at most $\eta \in [0, 1/2)$, which dramatically limits its power. In the adversarial noise, nevertheless, the adversary has the freedom to choose {\em any} {joint} distribution over the instance and label space.} Therefore, we consider an alternative yet fairly natural candidate:~we choose $g_t$ as the original gradient used in the Perceptron algorithm. That is, given the currently learned halfspace $w_{t-1}$ and a new labeled sample $(x_t, y_t)$, we set $g_t = - y_t x_t \cdot \ind{ y_t \neq \sign(w_{t-1} \cdot x_t)}$. It remains to develop a plausible sampling scheme so that a) $-g_t$ will have nontrivial correlation with the underlying halfspace $u$; and b) only most informative instances are sampled for labeling. To this end, we propose a new sampling region $X_{\hatw_{t-1}, b} := \{ x \in \Rd: 0 < \hatw_{t-1} \cdot x \leq b\}$, where $\hatw_{t-1} = \frac{w_{t-1}}{\twonorm{w_{t-1}}}$. Such time-varying region is different from active learning using empirical risk minimization~\cite{balcan2007margin,awasthi2017power,zhang2018efficient} as in these works $x$ is sampled from the full band $\abs{\hatw_0 \cdot x} \leq b$. On one hand, using our sampling region leads to a linear loss $\inner{w_{t-1}}{g_t}$ at most $O(b)$, while the band used by ERM would result in a loss of $O(b\sqrt{s}\log d)$. Observe that a tighter control on the loss implies tighter upper bound in \eqref{eq:omd}. On the other hand, we discover that by restricting on querying the label of  instances in $X_{\hatw_{t-1}, b}$, we are reducing the randomness of model updating because now $g_t = x_t \cdot \ind{y_t \neq 1}$, i.e. we update the model only when the returned label $y_t = -1$. It turns out that such {\em semi-randomness} facilitates our control of the correlation between each $-g_t$ and the underlying halfspace $u$. We note that the semi-random updating rule is inspired by \citet{yan2017revisiting}, where they used a much narrower sampling region $\big\{x: \frac{b}{2\sqrt{d}} \leq \hatw_{t-1} \cdot x \leq \frac{b}{\sqrt{d}}\big\}$ and a carefully rescaled Perceptron gradient $g_t = (\hatw_{t-1} \cdot x_t) \cdot  x_t \cdot \ind{y_t \neq 1}$ to accommodate their projection-free algorithm for learning under the uniform marginal distribution. In contrast, we incorporate different sampling strategy and gradients into (projected) mirror descent for PAC learning under isotropic log-concave marginal distributions.


\vspace{0.1in}
\noindent{\bfseries 3) A new characterization of the correlation between gradient and the underlying halfspace.} Our last ingredient is applying localization in the concept space~\cite{awasthi2017power}. Roughly speaking, before running online mirror descent, it is possible to construct an $\ell_2$-ball where the underlying halfspace $u$ resides in. Such trust region will be serving as the convex constraint set for online minimization. Using a novel analysis, we show that this interesting observation in allusion to the dedicated design of gradients implies that under the adversarial noise model,
\begin{equation}\label{eq:lower-cor}
\EXP_{ (x_t, y_t) \sim D_{\hatw_{t-1}, b}}[ \inner{u}{-g_t} ] \geq f_{u, b}(w_{t-1}) - \beta \cdot \theta(w_0, u),
\end{equation}
where we have the potential function
\begin{equation}
f_{u, b}(w_{t-1}) := \EXP_{ (x_t, y_t) \sim D_{\hatw_{t-1}, b}}\big[ \abs{ u \cdot x_t} \cdot \ind{ u \cdot x_t < 0}\big],
\end{equation}
and $\beta > 0$ is some quantity to be controlled. We argue that the function $f_{u, b}(w_{t-1})$ serves almost as a measure of  $\theta(w_{t-1}, u)$; hence combining it with \eqref{eq:omd} we have that the average of $\theta(w_{t-1}, u)$ is upper bounded by $\frac{1}{2} \cdot \theta(w_0, u)$. This observation, in conjunction with a non-standard online-to-batch conversion, results in the desired halfspace $\bar{w}$. We note two different aspects compared to \citet{zhang2020efficient}. First, our potential function $f_{u, b}$ is slightly distinct since we are considering a smaller sampling region. Second and more importantly, when deriving the lower bound for the correlation of $u$ and $-g_t$, we carry out a more involved analysis but still incur a negative penalty $-\beta \cdot \theta(w_0, u)$ resulted from the adversarial noise model. In contrast, this term does not appear in \citet{zhang2020efficient} (since the bounded noise model they studied is more benign). Manipulating the penalty turns out to be subtle since if the factor $\beta$ is large, there is no hope to upper bound the average of $f_{u, b}(w_{t-1})$ by $\frac{1}{2} \cdot \theta(w_0, u)$. We circumvent the technical issue by showing that $\beta$ is dominated by $\nu / b$ which is small as soon as $\nu \leq c_0 \epsilon$ for sufficiently small constant $c_0$ and $b$ is carefully chosen to be greater than $\epsilon$; see Lemma~\ref{lem:lower-correlation} in the appendix.

\subsection{Related works}

Label-efficient learning has also been broadly studied since gathering high quality labels is often expensive~\cite{cohn1994improving,dasgupta2005coarse,dasgupta2011active}. The prominent approaches include disagreement-based active learning~\cite{hanneke2011rates,hanneke2014theory}, margin-based active learning~\cite{balcan2007margin,balcan2013active,awasthi2015efficient}, selective sampling~\cite{cavallanti2011learning,dekel2012selective}, and adaptive one-bit compressed sensing~\cite{zhang2014efficient,baraniuk2017exponential}. There are also a number of interesting works that appeal to extra information to mitigate the labeling cost, such as comparison~\cite{xu2017noise,kane2017active,hopkins2020noise,zeng2020learning} and search~\cite{beygelzimer2016search}. 

Adversarial noise is closely related to the agnostic model, which was studied in \citet{haussler1992decision} and then coined out by~\citet{kearns1992toward}. Under the uniform marginal distributions, \citet{kalai2005agnostic} obtained the best error rate (see Section~\ref{subsec:agnostic} for a precise statement), though the running time and sample complexity is $O(d^{1/\epsilon^4})$. This bound has been proved almost best possible in very recent works under the statistical query model~\cite{diakonikolas2020near,goel2020statistical}. Interestingly, \citet{daniely2015ptas} characterized the tradeoff between the error rate and running time under the uniform marginal distribution by combining the techniques of polynomial regression~\cite{kalai2005agnostic} and localization~\cite{awasthi2017power}. By comprising on the error rate, \citet{klivans2009learning} presented an averaging-based algorithm and showed how to boost it to tolerate an adversarial noise rate $\nu = \tilde{\Omega}(\epsilon^3)$ in polynomial time when the marginal distribution is isotropic log-concave. Such noise tolerance has been improved by a series of recent ERM-based works \cite{balcan2009agnostic,beygelzimer2010agnostic,zhang2014beyond,awasthi2016learning,awasthi2017power,zhang2018efficient,diakonikolas2018learning}, among which $\nu = \Omega(\epsilon)$ is the best known noise tolerance. However, solving an ERM often requires more memory storage and computational cost per sample than online methods~\cite{shalev2007online}.




Achieving attribute efficiency has been a long-standing goal in machine learning and statistics~\cite{blum1990learning,blum1995learning}, and has been pursued in online classification~\cite{littlestone1987learning}, learning decision lists~\cite{servedio1999computational,klivans2004toward,long2006attribute}, compressed sensing~\cite{donoho2006compressed,candes2008introduction,tropp2010computational,shen2018tight},  one-bit compressed sensing~\cite{boufounos2008bit,plan2016generalized}, and variable selection~\cite{fan2001variable,fan2008high,zhang2010nearly,shen2017iteration,shen2017partial,wang2018provable}. It is worth mentioning that \citet{awasthi2016learning} gave a label-inefficient algorithm for uniformly learning sparse halfspaces with adversarial noise while this work and the closely related works consider non-uniform learning.

\paragraph{Roadmap.}
In Section~\ref{sec:setup}, we give preliminaries and collect the notations used in the paper. In Section~\ref{sec:alg}, we elaborate on our main algorithms. A theoretical analysis is given in Section~\ref{sec:guarantee}, along with a proof sketch of the main results. We conclude this paper in Section~\ref{sec:conc}, and defer the proof details to the appendix.

%% file: setup.tex
\section{Preliminaries}\label{sec:setup}

We study PAC learning of sparse homogeneous halfspaces with adversarial noise, where the instance space is $\Rd$, the label space is $\{-1, 1\}$, and the concept class is $\calC := \{ x \mapsto \sign(w \cdot x):  w\in \Rd,\ \twonorm{w} = 1,\ \zeronorm{w} \leq s\}$. Here, $\twonorm{w}$ denotes the $\ell_2$-norm and $\zeronorm{w}$ counts the number of non-zero elements in $w$. Observe that  we say a halfspace is non-sparse if $s = d$. An adversary $\oraclexy$ with adversarial noise works as follows: it first chooses an arbitrary joint distribution $D$ over $\Rd \times \{-1, 1\}$; the distribution $D$ is then fixed throughout learning. Let $D_X$ denote the marginal distribution over the instance space, which is promised to belong to a family of well-behaved distributions $\calD_{\calX}$; in this paper it is assumed to be isotropic log-concave (Assumption~\ref{as:x}).

A learner is given the instance and label space, the concept class $\calC$, the family of distributions $\calD_{\calX}$ (but not $D_X$), a target error rate $\epsilon \in (0, 1)$ and a failure confidence $\delta \in (0, 1)$, and the goal is to output in polynomial time a halfspace $\optu \in \calC$ such that with probability at least $1-\delta$, $\Pr_{(x, y) \sim D}(\sign( \optu \cdot x ) \neq \sign( u \cdot x)) \leq \epsilon$. In the passive learning setting, the learner is given access to a sample generation oracle $\oraclexy$ which returns a labeled sample $(x, y) \in \Rd \times \{-1, 1\}$ randomly drawn from the distribution $D$. Since we want to optimize the label complexity, we will consider a natural extension: when the learner makes a call to $\oraclexy$, a labeled sample $(x, y)$ is randomly drawn but only the instance $x$ is returned. The learner must make a separate call to a label revealing oracle $\oracley$ to obtain the label $y$. We refer to the total number of calls to $\oraclexy$ as the sample complexity of the algorithm, and that of $\oracley$ as the label complexity.

In our active learning algorithm, we will often want to draw instances from $D_X$ conditioned on a region $X_{\hatw, b} := \{x \in \Rd: 0 < \hatw \cdot x \leq b\}$ where $\hatw \in \Rd$ and $b > 0$ are given; this can be done by rejection sampling, where we repeatedly call $\oraclexy$ until seeing an instance $x$ that falls in $X_{\hatw, b}$. We will refer to $X_{\hatw, b}$ as sampling region, and $b$ is called band width. We denote by $D_{X| \hatw, b}$ (respectively $D_{\hatw, b}$) the distribution $D_X$ (respectively $D$) conditioned on the event that $x \in X_{\hatw, b}$.

Let $w$ be a vector in $\Rd$. We will frequently use $\hatw$ to denote its $\ell_2$-normalization $\frac{w}{\twonorm{w}}$. For a scalar $\gamma \geq 1$, we denote by $\norm{w}_{\gamma}$ the $\ell_{\gamma}$-norm of $w$. Let $s > 0$ be an integer less than $d$. The hard thresholding operation $\calH_s(w)$ zeros out all but the $s$ largest (in magnitude) entries in $w$. For two vectors $w$ and $w'$, we write $\theta(w, w')$ for the angle between them.

We reserve $p$ and $q$ for specific values: $p = \frac{\ln(8d)}{\ln(8d)-1}$ and $q = \ln(8d)$ (note that $\frac1p + \frac1q = 1$). We will use the $\ell_p$-norm regularizer, that is, $\Phi(w) = \frac{1}{2(p-1)} \pnorm{w - v}^2$ for some given vector $v$. It is known that $\Phi(w)$ is $1$-strongly convex with respect to the $\ell_p$-norm~\citep[Lemma~17]{shalev2007online}. We denote the Bregman divergence induced by $\Phi(w)$ by $\breg(w; w') := \Phi(w) - \Phi(w') - \inner{\nabla \Phi(w')}{w - w'}$. Observe that $\breg(w; v) = \Phi(w)$ where $v$ is the reference vector appearing in $\Phi(w)$.

We will sometimes phrase our theoretical guarantee in terms of angles between two halfspaces. The following lemma, due to \citet{balcan2013active}, is useful to convert the guarantee of angles to that of error rate.
\begin{lemma}\label{lem:error=angle}
There exists an absolute constant $\bar{c} > 0$ such that the following holds. Let $D_X$ be an isotropic log-concave distribution. For any two vectors $w$ and $w' \in \Rd$, $\Pr_{x \sim D_X}(\sign(w \cdot x) \neq \sign(w' \cdot x)) \leq \bar{c} \cdot \theta(w, w')$.
\end{lemma}

We remark that a closely related noise model is the agnostic model~\cite{haussler1992decision,kearns1992toward,kalai2005agnostic}, where only Assumption~\ref{as:x} is satisfied and the goal is to output a halfspace that approximates the best halfspace in $\calC$. The results from our theorems under the adversarial noise model translate immediately into PAC guarantees for the agnostic model; see Section~\ref{subsec:agnostic} for more details.

%% file: algorithm.tex
\section{Main Algorithm}\label{sec:alg}

We present our online active learning algorithm in Algorithm~\ref{alg:main}, which consists of two major stages:~initialization and refinement. In the initialization stage, the goal is to find a halfspace $v_0$ that has a constant acute angle with the underlying halfspace $u$. It will then be used as a warm start for the refinement stage, where the procedure \refine is repeatedly invoked to cut off the angle with $u$ by half in each phase $k$. Therefore, after $K$ phases of refinement, we will obtain a halfspace $v_K$ satisfying $\theta(v_K, u) = 2^{-K} = O(\epsilon)$, which by Lemma~\ref{lem:error=angle} implies that $v_K$ has small error rate with respect to the underlying halfspace $u$ defined in Assumption~\ref{as:y}.

Since \init invokes \refine as well, we will introduce the latter first. Generally speaking, the \refine algorithm, i.e. Algorithm~\ref{alg:refine}, belongs to the family of online mirror descent algorithms with Perceptron gradient and $\ell_p$-norm regularization. The crucial ingredients that make it attribute and label efficient are a carefully crafted constraint set, a time-varying sampling region, and semi-random gradients, as we described in Section~\ref{subsec:technique}. In particular, the constraint set $\calK$ is constructed in such a way that the underlying halfspace $u$ is guaranteed to stay in it (with overwhelming probability). Thus, it serves as a trust region into which all the iterates $w_t$ are projected back. It is worth mentioning that we did not put an $\ell_1$-norm constraint in $\calK$; this is because we already have utilized the $\ell_p$-norm regularization to simultaneously guarantee attribute efficiency~\cite{grove2001general,gentile2003robustness} and the stability of online minimization~\cite{shalev2012online,orabona2019modern}. The second component in \refine is the time-varying sampling region $X_{\hatw_{t-1}, b} = \{x \in \Rd: 0 < \hatw_{t-1} \cdot x \leq b\}$, which results in a linear loss as small as $O(b)$ in each iteration. In contrast, a naive online approach to simulate the ERM algorithm of \citet{zhang2018efficient} would lead to a loss as large as $O(b \sqrt{s} \log d)$. The idea of using time-varying sampling paradigm has appeared in a few online active learning algorithms \cite{yan2017revisiting,zhang2020efficient}. Ours is less restrictive than the one in \citet{yan2017revisiting}, and is more dedicated to the much more challenging adversarial noise compared to the bounded noise model considered in \citet{zhang2020efficient}. Along with the new sampling region is a semi-random Perceptron gradient. Recall that the original gradient used in Perceptron is given by $g_t = - y_t x_t \cdot \ind{y_t \neq \sign(w_{t-1} \cdot x_t)}$. Since $x_t \in X_{\hatw_{t-1}, b}$, we update the model only when the label $y_t$ returned by the adversary equals $-1$, thus inducing the gradient displayed in Algorithm~\ref{alg:refine}. Finally, after running mirror descent for $T$ iterations, we perform an averaging scheme followed by hard thresholding, to ensure that the output $\tilde{w}$ belongs the to concept class.  We remark that the running time of \refine is polynomial in $d$, since in each iteration $t$, updating the model requires solving a convex program. Regarding the label complexity, it is exactly equal to the total iteration number $T$. We also remark that obtaining $x_t$ can be done by calling $\oraclexy$ for $O(1/b)$ times since the probability mass of $X_{\hatw_{t-1}, b}$ on $D_X$ is $\Theta(b)$; see Lemma~\ref{lem:logconcave}.

\begin{algorithm}[t]
\caption{Main Algorithm}
\label{alg:main}
\begin{algorithmic}[1]
\REQUIRE Target error rate $\epsilon \in (0, 1)$, failure probability $\delta \in (0, 1)$, sparsity $s$.
\ENSURE Halfspace $\optu \in \Rd$ such that $\Pr_{(x, y) \sim D}(\sign(\optu \cdot x) \neq \sign(u \cdot x)) \leq \epsilon$.

\STATE ${v}_0 \gets \init(\frac{\delta}{2},  s)$.

\STATE $K \gets \lceil \log \frac{\bar{c} \pi}{8 \epsilon} \rceil$ where $\bar{c}$ is defined in Lemma~\ref{lem:error=angle}.
\label{line:init}
\FOR{$k=1,\ 2,\ \dots,\ K$}
\label{line:refine-start}
\STATE $v_k \leftarrow \refine(v_{k-1}, \frac{\delta}{2 k(k+1)}, s, \alpha_k, b_k, \calK_k, \Phi_k, T_k)$, where the
step size $\alpha_k = \tilde{\Theta}\del[2]{ 2^{-k} \cdot \del[1]{\log \frac{d \cdot k^2 \cdot 2^k}{\delta}}^{-2}}$, band width $b_k = \Theta\del[1]{ 2^{-k}}$, constraint set
\begin{equation*}
\calK_k = \cbr[1]{w: \twonorm{ w - v_{k-1} } \leq \pi\cdot 2^{-k-2},\ \twonorm{w} \leq 1 },
\end{equation*}
regularizer $\Phi_k(w) = \frac{1}{2(p-1)}\pnorm{w - v_{k-1}}^2$, number of iterations $T_k = \tilde{O}\del[2]{ s \log d \cdot \del[1]{\log \frac{d \cdot k^2 \cdot 2^k}{\delta }}^2 }$. \label{line:refine}
\ENDFOR
\label{line:refine-end}
\STATE {\bfseries return} $\optu \gets v_{K}$. 
\end{algorithmic}
\end{algorithm}

\begin{algorithm}[t]
\caption{\init}
\label{alg:init}
\begin{algorithmic}[1]
\REQUIRE Failure probability $\delta'$, sparsity  $s$.
\ENSURE An $s$-sparse halfspace ${v}_0$ such that $\theta({v}_0, u) \leq \frac{\pi}{8}$.
\STATE $(x_1,y_1),\ldots,(x_m, y_m) \gets$ call $\oraclexy$ to draw $m$ instances, and query $\oracley$ for their labels, where $m =  O(s \log \frac{d}{\delta'})$.
\STATE Compute $w_{\text{avg}} = \frac 1 m \sum_{i=1}^m  y_i x_i$.
\label{line:init-avg}
\STATE Let $\wsharp = \frac{\calH_{\tilde{s}}(w_{\text{avg}})}{\| \calH_{\tilde{s}}(w_{\text{avg}}) \|}$, where $\tilde{s} = {81\cdot 2^{40}}s$.
\label{line:init-ht}
\STATE Let $\calK = \cbr[1]{w \in \Rd: \twonorm{w} \leq 1,\ {w} \cdot {\wsharp} \geq \frac{1}{9 \cdot 2^{20}}}$ and find a point $w_0 \in \calK \cap \{ w\in \Rd: \onenorm{w} \leq \sqrt{s} \}$.
\STATE {\bfseries return} ${v}_0 \gets \refine(w_0, \frac{\delta'}{2},  s, \alpha, b, \calK, \Phi, T)$, where step size $\alpha = \tilde{\Theta}\del[1]{ \log^{-2}\frac{d}{\delta'} }$, band width $b = \frac{1}{81\cdot 2^{22}}$,
regularizer $\Phi(w) = \frac{1}{2(p-1) }\pnorm{w - w_0}^2$, and number of iterations $T = \tilde{O}\del[2]{ s \log d \cdot \del[1]{\log\frac{d}{\delta'}}^2}$.
\end{algorithmic}
\end{algorithm}

Now we elaborate on the \init algorithm, namely Algorithm~\ref{alg:init}. Technically speaking, one important condition for the success of our analysis is that an overwhelming portion of the iterates must have acute angles with the underlying halfspace $u$. Therefore, the hypothesis testing approach proposed in \citet{awasthi2017power} does not work out in our case since we will lose control of the intermediate iterates. To circumvent the technical challenge, we tailor the averaging-based initialization scheme of \citet{zhang2020efficient} to the adversarial noise model. It is possible to show that as far as the noise rate $\nu$ is low, $\wavg$ has a positive correlation with $u$, and performing hard thresholding almost preserves it, i.e. ${u} \cdot {\wsharp} = \Omega(1)$. Therefore, we obtain a good reference vector $\wsharp$ which is guaranteed to have an acute angle with $u$. Based on an enhanced constraint set that takes the correlation into consideration, we are able to show that most of the iterates are admissible, and hence our analysis of the \refine algorithm can be reused to show that the output $v_0$ will have a small acute angle with $u$. Note that we are not making efforts to optimize the constants that appear in the \init algorithm; in practice, we believe that our algorithm works under reasonable constants. It is also worth mentioning that \citet{zhang2021improved} recently developed a simpler initialization scheme for the problem of learning halfspaces with Massart or Tsybakov noise; it will be interesting to adapt their approach to the adversarial noise model as a future work.

\begin{algorithm}[t]
\caption{\refine}
\label{alg:refine}
\begin{algorithmic}[1]
\REQUIRE Initial $s$-sparse halfspace $w_0$, failure probability $\delta'$, sparsity $s$, step size $\alpha$, band width $b$, convex constraint set $\calK$, regularization function $\Phi: \Rd \rightarrow [0, +\infty)$, number of iterations $T$.
\ENSURE Refined $s$-sparse halfspace $\tilde{w}$ such that $\theta(\tilde{w}, u) \leq \frac12 \cdot \theta(w_0, u)$.
\FOR{$t=1, 2, \dots, T$}

\STATE  Call $\oraclexy$ to obtain an instance $x_t$ in $X_{\hatw_{t-1}, b}$, and query $\oracley$ for its label $y_t$ (recall that $\hatw_{t-1}$ is the $\ell_2$-normalization of $w_{t-1}$).\label{line:sample}

\STATE  $w_t \gets\argmin_{w \in \calK} \inner{w}{\alpha g_t} + \breg(w; w_{t-1})$, where $g_t =  x_t \cdot \ind{y_t = -1}$.\label{line:omd}

\ENDFOR

\STATE $\bar{w} \gets \frac{1}{T} \sum_{t=1}^T \hat{w}_t$.
\STATE {\bfseries return} $\tilde{w} \leftarrow \frac{\calH_s(\bar{w}) }{ \twonorm{ \calH_s(\bar{w})} }$.
\end{algorithmic}
\end{algorithm}

%% file: guarantee.tex
\section{Performance Guarantee}\label{sec:guarantee}

We are now in the position to state our main theorem, which is a formal statement of Theorem~\ref{thm:sparse:informal}.

\begin{theorem}[Main result]\label{thm:main}
Suppose that Assumptions~\ref{as:x} and \ref{as:y} are satisfied. If $\nu \leq c_0 \epsilon$ for some small absolute constant $c_0 > 0$, then with probability $1 - \delta$, the output of Algorithm~\ref{alg:main}, $\optu$, satisfies $\Pr_{(x, y) \sim D}\big( \sign(\optu\cdot x) \neq \sign(u \cdot x)\big) \leq \epsilon$. Moreover, the running time is $\poly(d, \frac{1}{\epsilon}, \log\frac{1}{\delta})$, the label complexity is $\tilde{O}\big(s \cdot \polylog(d, \frac{1}{\epsilon}, \frac{1}{\delta})\big)$, and the sample complexity is $\tilde{O}\big(\frac{s}{\epsilon} \cdot \polylog(d, \frac{1}{\delta})\big)$.
\end{theorem}

\begin{remark}
A more concrete label and sample complexity reads as $\tilde{O}(s \log d \cdot \log^3\frac{d}{\epsilon \delta})$ and $\tilde{O}\big( \frac{s}{\epsilon} \cdot \log^4 \frac{d}{\delta} \big)$, respectively; see Theorem~\ref{thm:main-refine-restate} and Theorem~\ref{thm:main-init-restate} in the appendix.
\end{remark}

\begin{remark}[Excess risk]
By the triangle inequality, we have $\err_D(\optu) - \err_D(u) \leq  \Pr_{(x, y) \sim D}( \sign(\optu\cdot x) \neq \sign(u \cdot x)) \leq  \epsilon$. Namely, the excess risk of $\optu$ with respect to $u$ is at most $\epsilon$ over the underlying distribution.
\end{remark}

\begin{remark}[Implications to passive learning]
It is possible to convert our algorithm to an online passive learner, where there is only one oracle $\oraclexy$ that always returns a labeled instance upon request. To this end, observe that the active learner interacts with the oracle exclusively in Step~2 of \refine. Therefore, we only need to modify this step as follows in the passive setting: repeatedly call $\oraclexy$ to obtain a sequence of labeled instances $\{ x_i, y_i \}_{i \geq 1}$ until seeing a pair $(x_t, y_t)$ such that $x_t \in X_{\hatw_{t-1}, b}$. Then we use $(x_t, y_t)$ to update the classifier in Step~3 (rather than using all the labeled instances that are drawn from $\oraclexy$). It is easy to see that the label and sample complexity of the passive learner are both $\tilde{O}\big(\frac{s}{\epsilon}\cdot \polylog(d, \frac{1}{\delta})\big)$.
\end{remark}

The following corollary, which is a formal statement of Theorem~\ref{thm:non-sparse:informal}, concerns learning of non-sparse halfspaces and is an immediate application of Theorem~\ref{thm:main} by setting $s = d$.

\begin{corollary}
Assume same conditions as in Theorem~\ref{thm:main}. With probability $1 - \delta$, $\Pr_{(x, y) \sim D}\big( \sign(\optu\cdot x) \neq \sign(u \cdot x)\big) \leq \epsilon$. Moreover, the running time is $\poly(d, \frac{1}{\epsilon}, \log\frac{1}{\delta})$, the label complexity is $\tilde{O}\big(d \cdot \polylog(\frac{1}{\epsilon}, \frac{1}{\delta})\big)$, and the sample complexity is $\tilde{O}\big(\frac{d}{\epsilon} \cdot \polylog(\frac{1}{\delta})\big)$.
\end{corollary}

\subsection{Implications to agnostic learning}\label{subsec:agnostic}

In the agnostic model~\cite{haussler1992decision,kearns1992toward}, the adversary chooses a joint distribution $D$ over $\Rd \times \{-1, 1\}$ and fixes it throughout the learning process. Let $\OPT = \min_{w \in \calC} \err_D(w)$. The goal of the learner is to output a hypothesis $\optu$ such that $\err_D(\optu) \leq c \cdot \OPT + \epsilon$ for some approximation factor $c \geq 1$. The crucial difference from the adversarial noise model is that now Assumption~\ref{as:y} may not be satisfied (in other words, $\OPT$ can be very large compared to the target error rate $\epsilon$).

\citet{kalai2005agnostic} developed a polynomial regression algorithm that achieves approximation guarantee with $c = 1$, where the computational and sample complexity are both $O(d^{2^{\poly(1/\epsilon)}})$ for learning under isotropic log-concave distributions and are $O(d^{1/\epsilon^4})$ for uniform distributions.

On the other side, \citet{kalai2005agnostic} and a number of recent works also obtained weaker (yet still quite nontrivial) approximation guarantee of $O(\OPT) + \epsilon$ with running time and sample complexity polynomial in $d$ and $\frac{1}{\epsilon}$; see, for example, \citet{awasthi2017power,zhang2018efficient,diakonikolas2018learning,diakonikolas2020non}. We follow this line, and remark that our main result, Theorem~\ref{thm:main}, can be immediately translated into the constant approximation guarantee under the agnostic model. In fact, Lemma~C.1 of \citet{awasthi2017power} made an interesting observation that if an algorithm can tolerate adversarial noise of $\nu = \Omega(\epsilon)$, then it essentially achieves error rate of $O(\OPT) + \epsilon$ in the agnostic model, with the same running time, label complexity, and sample complexity (up to a constant multiplicative factor). We therefore have the following result, which is a formal statement of Theorem~\ref{thm:agnostic:informal}, by combining Theorem~\ref{thm:main} we established and Lemma~C.1 of \citet{awasthi2017power}; we omit the proof since it is fairly straightforward.

\begin{corollary}\label{coro:agnostic}
Suppose that Assumption~\ref{as:x} is satisfied. Then with probability $1 - \delta$, $\err_D(\optu) \leq c \cdot \OPT + \epsilon$ for some absolute constant $c > 1$. Moreover, the running time is $\poly(d, \frac{1}{\epsilon})$, the label complexity is $\tilde{O}\big(s \cdot \polylog(d, \frac{1}{\epsilon}, \frac{1}{\delta})\big)$, and the sample complexity is $\tilde{O}\big(\frac{s}{\epsilon} \cdot \polylog(d, \frac{1}{\delta})\big)$.
\end{corollary}

\subsection{Proof of Theorem~\ref{thm:main}}
Theorem~\ref{thm:main} hinges on the following two important results, characterizing the performance of \init and \refine respectively. A more precise statement and a detailed proof can be found in Appendix~\ref{app:sec:proof} in the supplementary material.

\begin{theorem}\label{thm:main-init}
Consider the \init  algorithm. If Assumptions~\ref{as:x} and \ref{as:y} are satisfied and $\nu \leq c_o \epsilon$, then with probability $1-\delta'$, the output of \init, $v_0$, is such that $\theta(v_0, u) \leq \frac{\pi}{8}$. The running time is $\poly(d, \log\frac{1}{\delta'})$, the label complexity is $\tilde{O}\big( s \cdot \polylog(d, \frac{1}{\delta'}) \big)$, and the sample complexity is $\tilde{O}\big( s \cdot \polylog(d) \big)$.
\end{theorem}

\begin{theorem}\label{thm:main-refine}
Consider the \refine algorithm. Suppose that Assumptions~\ref{as:x} and \ref{as:y} are satisfied and $\nu \leq c_o \epsilon$. Then with probability $1-\delta'$, the output of the \refine algorithm, $\tilde{w}$, satisfies $\theta(\tilde{w}, u) \leq \frac{1}{2}\cdot \theta(w_0, u)$. The running time of \refine is $T \cdot \poly(d, \frac{1}{b}, \log\frac{1}{\delta'})$, the label complexity is $T$, and the sample complexity is $O(T/b + T \log\frac{T}{\delta'})$ where $T$ and $b$ are the inputs to \refine.
\end{theorem}

We first explain the results in Theorem~\ref{thm:main-refine}. The label complexity of \refine is straightforward since we request one label per iteration, and the total number of iterations is $T$. We start with the analysis of the sample complexity of \refine, i.e. the number of calls to $\oraclexy$ in Step~\ref{line:sample} therein. Since the marginal distribution is assumed to be isotropic log-concave, Lemma~\ref{lem:logconcave} shows that $\Pr_{x_t \sim D_X}( x_t \in X_{\hatw_{t-1}, b}) \geq c_2 b$ for some absolute constant $c_2 > 0$. Thus, by Chernoff bound, we need to call $\oraclexy$ for $O(\frac{1}{b} + \log\frac{T}{\delta'})$ times in order to obtain one $x_t$ in the band with probability $1 - \frac{\delta'}{2T}$. Thus, by union bound over the $T$ iterations in \refine, with probability $1 - \frac{\delta'}{2}$, the total number of calls to $\oraclexy$ is $O(\frac{T}{b} + T \log\frac{T}{\delta'})$. Note that this is also the computational cost for sampling. On the other side, updating the iterates, i.e. Step~\ref{line:omd}, involves solving a convex program in $\Rd$, which has a running time polynomial in $d$ per iteration. Thus, the overall computational cost of \refine is $T \cdot \poly(d, \frac{1}{b}, \log\frac{1}{\delta'})$.

Now we explain the results in Theorem~\ref{thm:main-init}. Generally speaking, the \init algorithm consists of two major steps, one for constructing the constraint set $\calK$ and one for obtaining a good initial halfspace $v_0$ based on $\calK$. To obtain $\calK$, it consumes $m$ labeled instances and the computational cost for sampling them is $O(m)$ since rejection sampling is not needed. Then it aims to find a point $w_0$ in a convex set, which is polynomial-time solvable; in fact, we can set $w_0$ to the zero vector and then project it onto $\calK$, corresponding to solving a convex program. The second major step is to invoke \refine, for which we have just analyzed. Combining these observations and the parameters specified in the \init algorithm, we obtain the announced results.

\begin{proof}[Proof of Theorem~\ref{thm:main}]
First, Theorem~\ref{thm:main-init} implies $\theta(v_0, u) \leq \frac{\pi}{8}$ with probability $1-\frac{\delta}{2}$. In addition, for any phase $k$ in Algorithm~\ref{alg:main}, we specify in Theorem~\ref{thm:main-refine} that $w_0 = v_{k-1}$ and $\tilde{w} = v_k$, and obtain that $\theta(v_k, u) \leq \frac{1}{2} \cdot \theta(v_{k-1}, u)$ with probability $1-\frac{\delta}{2k(k+1)}$. By telescoping, we get $\theta(v_K, u) \leq 2^{-K} \cdot \frac{\pi}{8} \leq \epsilon / \bar{c}$ in light of our setting of $K$; this inequality holds with probability $1 - \frac{\delta}{2} - \sum_{k=1}^{K} \frac{\delta}{2k(k+1)} \geq 1 - \delta$ by union bound. This in allusion to Lemma~\ref{lem:error=angle} gives the desired error rate of $\optu = v_K$ with respect to $u$.

The running time of Algorithm~\ref{alg:main} follow from those we analyzed for \init and \refine, and from the hyper-parameter settings on $b_k$, $T_k$, and $\delta_k$ in each phase $k$. In particular, observe that $b_k \geq \epsilon$ for all $k \leq K$. Therefore, the running time is given by $\poly(d, \log\frac{1}{\delta}) + \sum_{k=1}^{K} T_k \cdot \poly(d, \frac{1}{b_k}, \log\frac{k^2}{\delta}) = \poly(d, \frac{1}{\epsilon}, \log\frac{1}{\delta})$. 

Likewise, for label complexity and sample complexity, we can add up the cost in the initialization stage and that of the $K$ phases of refinement to obtain the bounds as claimed.
\end{proof}

%% file: conclusion.tex
\section{Conclusion and Future Works}\label{sec:conc}

This paper studies the fundamental problem of learning halfspaces with adversarial noise. We have presented the first attribute-efficient, label-efficient, and noise-tolerant algorithm in the online setting, under the general isotropic log-concave marginal distributions. Prior to this work, existing online learners are either subject to label inefficiency or suboptimal noise tolerance, or work under restrictive marginal distributions. We have shown that our label and sample complexity are near-optimal, and the learner achieves PAC guarantee in polynomial time. Prior to this work, such performance guarantee is only achieved by a very recent batch algorithm. Our results also have immediate implications to the agnostic model, and match the best known results obtained by polynomial-time batch algorithms.

We discuss a few important directions for future investigation. First, it is interesting to develop online PAC algorithms with $\OPT + \epsilon$ approximation error under the agnostic model, by leveraging, for example, the polynomial regression technique~\cite{kalai2005agnostic} into the online mirror descent framework. Second, it is useful to extend the analysis to more general concept classes such as intersections of halfspaces~\cite{klivans2002learning,diakonikolas2018learning}. It will also be important to design PAC algorithms that leverage additional types of queries such as pairwise comparison~\cite{kane2017active,xu2017noise} in the scenario where labels are extremely demanding (e.g. medical data), or to develop new projection-free algorithms for even faster computation.

%

%% file: appendix.tex
\section{Proof Details}\label{app:sec:proof}

From now on, we always implicitly require that Assumption~\ref{as:x} and Assumption~\ref{as:y} hold if not specified. Since the \init algorithm also depends on \refine, we will start our analysis with the latter.

First of all, we recall a few important notations. Given $w \in \Rd$, the sampling region for the unlabeled example is a band, which is given by 
\begin{equation}
X_{\hatw, b} := \{x \in \Rd: 0 < \hatw \cdot x \leq b\},\ \text{where}\ \hatw := \frac{w}{\twonorm{w}}.
\end{equation}
We denoted by $D_{X| \hatw, b}$ the distribution of $D_X$ conditioned on $x \in X_{\hatw, b}$, and by $D_{\hatw, b}$ the distribution of $D$ conditioned on $x \in X_{\hatw, b}$.

Recall that given $w \in \Rd$ and $x \sim D_{X| \hatw, b}$, our prediction $\haty = \sign(w \cdot x) = 1$ and we set the gradient in \refine as 
\begin{equation*}
g = x \cdot \ind{y= -1},
\end{equation*}
where $y$ is the label returned by the adversary. Also recall that we defined the potential function
\begin{equation*}
f_{u, b}(w) = \EXP_{ x \sim D_{X|\hatw, b}} \sbr{ \abs{u \cdot x} \cdot \ind{ u \cdot x < 0 } }.
\end{equation*}

Finally, we note that the capital letters $C$ and $K$, and their subscript variants such as $C_1$ and $K_1$, are used to denote absolute constants whose values may differ from appearance to appearance. However, we reserve $c_0$, $c_1$, and $c_2 > 0$ for specific absolute constants: $c_0$ is a sufficiently small constant such that the noise rate $\nu \leq c_0 \epsilon$, $c_1$ and $c_2$ are specified in Lemma~\ref{lem:E[wx^2]} and Lemma~\ref{lem:logconcave} respectively.

\subsection{Analysis of \refine}

Intuitively, since we are performing gradient descent, we would hope that the negative gradient has a nontrivial correlation with the underlying halfspace $u$. The following lemma formalizes the intuition.

\begin{lemma}\label{lem:lower-correlation}
Given $w \in \Rd$ with $\twonorm{w - u} \leq r$ and $b > 0$, let $g = x \cdot \ind{y = -1}$ where $(x, y) \sim D_{\hatw, b}$. If the noise rate $\nu \leq c_0 b$ for some absolute constant $c_0 > 0$, then
\begin{equation*}
\EXP\sbr{ u \cdot (-g) } \geq f_{u, b}(w) - \sqrt{\frac{c_0 c_1}{c_2} } \cdot (b+r),
\end{equation*}
where the expectation is taken over the random draw of $(x, y)$.
\end{lemma}
\begin{proof}
By the definition of $g$, it follows that
\begin{equation}\label{eq:1}
\EXP\sbr{ u \cdot (-g) } = \EXP\sbr{ - (u \cdot x) \cdot \ind{y = -1} }.
\end{equation}
As we can rewrite the indicator function $\ind{y = -1}$ in an equivalent form as follows:
\begin{equation}
\ind{y = -1} = \ind{ u \cdot  x < 0} + \ind{u \cdot x > 0, y = -1} - \ind{y = 1, u \cdot x < 0},
\end{equation}
\eqref{eq:1} can be written as
\begin{align*}
\EXP\sbr{ u \cdot (-g) } = \underbrace{\EXP \sbr{ - (u \cdot x) \cdot \ind{u \cdot x < 0 } } }_{E_1}  + \underbrace{\EXP \sbr{ - (u \cdot x) \cdot  \del{ \ind{u \cdot x > 0, y = -1} - \ind{y = 1, u \cdot x < 0} }  } }_{E_2}.
\end{align*}
First, we argue that $E_1 = f_{u, b}(w)$. In fact, when $u \cdot x \geq 0$, $E_1 = 0 = f_{u, b}(w)$; when $u \cdot x < 0$, $E_1 = \EXP[ \abs{u \cdot x} ] = f_{u, b}(w)$.

Let us now consider the term $E_2$, whose absolute value can be bounded by
\begin{equation*}
\abs{E_2} \leq \EXP \sbr{ \abs{u \cdot x}  \cdot  \ind{\sign(u \cdot x) \neq {y}}   } \leq  \sqrt{ \EXP \sbr{(u \cdot x)^2} \cdot \EXP \sbr{ \ind{\sign(u \cdot x) \neq {y}}} }.
\end{equation*}
By Lemma~\ref{lem:E[wx^2]}, $\EXP\sbr{ (u \cdot x)^2 } \leq c_1 (b^2 + r^2)$. On the other side, from the definition of $\nu$-adversarial noise, we have
\begin{equation*}
\EXP \sbr{ \ind{\sign(u \cdot x) \neq {y}}} = \Pr_{(x, y) \sim D_{\hatw, b}}( y \neq \sign(u \cdot x)) \leq \frac{\Pr_{(x, y) \sim D}( y \neq \sign(u \cdot x) )}{ \Pr_{x \sim D_X}(  x \in X_{\hatw, b} ) } \leq   \frac{\nu}{c_2 \cdot b}.
\end{equation*}
In the first inequality of the above expression, we use the fact that for an event $A$, $\Pr_{(x, y) \sim D_{\hatw, b}}(A) = \Pr_{(x, y) \sim D}(A \mid x \in X_{\hatw, b} ) \leq \frac{\Pr_{(x, y) \sim D}(A)}{\Pr_{x \sim D_X}( x \in X_{\hatw, b}  )}$. In the second inequality, we use Lemma~\ref{lem:logconcave} to bound the denominator from below.

Therefore,
\begin{equation*}
\abs{E_2} \leq \sqrt{ c_1 (b^2 + r^2) \cdot \frac{\nu}{c_2 \cdot b} } \leq \sqrt{ \frac{c_0 c_1}{c_2} (b^2 + r^2)} \leq \sqrt{\frac{c_0 c_1}{c_2}} \cdot (b+ r).
\end{equation*}
Now combining the above estimate and that of $E_1$, we prove the lemma.
\end{proof}

\begin{lemma}\label{lem:omd}
There exists an absolute constant $C > 0$ such that the following holds.
Suppose the algorithm \refine is run with initialization $w_0$, step size $\alpha > 0$, bandwidth $b > 0$, convex constraint set $\calK$, regularizer $\Phi(w) = \frac{1}{2(p-1)} \pnorm{w - w_0}^2$, number of iterations $T$, where the following are satisfied:
\begin{enumerate}
\item $\onenorm{w_0 - u} \leq \rho$;
\item $w_0 \in \calK$ and $u \in \calK$;
\item for all $w \in \calK$, $\twonorm{w - u} \leq r$.
\end{enumerate}
Then, with probability $1-\delta$,
\begin{equation*}
C \cdot \frac{1}{T} \sum_{t=1}^{T} f_{u, b}(w_{t-1}) \leq   (b+r) \del{ \frac{\sqrt{\log(1/\delta)}}{\sqrt{T}} + \frac{\log(1/\delta)}{T} + C \sqrt{\frac{c_0 c_1}{c_2}} } +  \frac{\rho^2 \log d}{ \alpha T} + \alpha  \cdot  \log^2\frac{Td}{b \delta}.
\end{equation*}
\end{lemma}
\begin{proof}
By standard analysis of online mirror descent (see, e.g. Theorem~6.8 of \citet{orabona2019modern}), we have
\begin{equation*}
\alpha \sum_{t=1}^{T} w_{t-1} \cdot g_t - \alpha \sum_{t=1}^{T} u \cdot g_t \leq \breg(u, w_0) + \sum_{t=1}^{T} \qnorm{\alpha g_t}^2.
\end{equation*}
Since $w_{t-1} \cdot g_t = w_{t-1} \cdot x_t \cdot \ind{ y_t = -1}$ and $x_t$ is such that $\hatw_{t-1} \cdot x_t > 0$, we have $w_{t-1} \cdot g_t \geq 0$ for all $t$. Using this observation and dividing both sides by $\alpha$, we obtain
\begin{equation}\label{eq:tmp:omd-regret}
\sum_{t=1}^{T} u \cdot (-g_t) \leq \frac{\breg(u, w_0)}{\alpha} + \alpha \sum_{t=1}^{T} \qnorm{g_t}^2.
\end{equation}
We first present upper bounds for the right-hand side. In particular, note that
\begin{equation}\label{eq:tmp:breg}
\breg(u, w_0) = \frac{1}{2(p-1)} \pnorm{u - w_0}^2 \leq \frac{\ln(8d) - 1}{2} \rho^2 \leq \frac{\rho^2 \ln(8d)}{2},
\end{equation}
where in the first inequality we use the fact that for any $p > 1$, $\pnorm{u - w_0} \leq \onenorm{u - w_0}$.

For the $\ell_q$-norm of $g_t$, denote $g_t^{(j)}$ the $j$th coordinate of $g_t$. We have
\begin{equation}
\norm{g_t}_q = \del[2]{ \sum_{j=1}^{d} \abs{g_t^{(j)}}^q }^{1/q} \leq ( d \infnorm{g_t}^q )^{1/q} \leq 2 \infnorm{g_t} \leq 2 \infnorm{x_t},
\end{equation}
where the first inequality makes use of the definition of $\ell_{\infty}$-norm, the second inequality applies the setting $q = \ln(8d)$, and the last step follows from the setting of $g_t$. On the other side, it is known that for any $x_t \sim D_{X| \hatw_{t-1}, b}$, with probability $1 - \frac{\delta}{2T}$, we have $\infnorm{x_t} \leq K_1 \cdot \log\frac{Td}{b \delta}$ for some absolute constant $K_1 > 0$; see Lemma~\ref{lem:|x|_inf-D_ub}. Hence, the union bound implies that with probability $1-\frac{\delta}{2}$, $\max_{1 \leq t \leq T} \infnorm{x_t} \leq K \cdot \log\frac{Td}{b \delta}$, which further gives 
\begin{equation}\label{eq:tmp:g_t}
\max_{1 \leq t \leq T} \qnorm{g_t} \leq 2K_1 \cdot \log\frac{Td}{b \delta}.
\end{equation}

Now we consider a lower bound of $\sum_{t=1}^{T} u \cdot (-g_t)$. Define the filtration 
\[
\calF_{t} := \sigma(w_0, x_1, y_1, w_1, \dots, x_{t-1}, y_{t-1}, w_{t}),
\]
and denote by $\EXP_{t-1}[\cdot]$ the expectation over $(x_t, y_t) \sim D_{\hatw_{t-1}, b}$ conditioning on the past filtration $\calF_{t-1}$; likewise for $\Pr_{t-1}(\cdot)$.

By existing tail bound of one-dimensional isotropic log-concave distributions in the band $X_{\hatw_{t-1}, b}$ (see, e.g. Lemma~3.3 of \citet{awasthi2017power}), and the fact that $\twonorm{u - \hatw_{t-1}} \leq 2 \twonorm{u - w_{t-1}} \leq 2 r$, we have
\begin{equation*}
\Pr_{t-1}\big( \abs{u \cdot x_t} \geq a \big) \leq K \exp\del[2]{- K' \cdot \frac{a}{2r + b}},
\end{equation*}
for some constants $K, K' > 0$, implying that
\begin{equation*}
\Pr_{t-1}\big( \abs{u \cdot (-g_t)} \geq a \big) \leq K \exp\del[2]{- K' \cdot \frac{a}{2r + b}}.
\end{equation*}
Now applying Lemma~\ref{lem:subexp-tail-hoeff} with $Z_t = u \cdot (-g_t)$ therein gives that with probability $1 - \frac{\delta}{2}$,
\begin{equation}\label{eq:tmp:martingale}
\abs{\sum_{t=1}^{T} u \cdot (-g_t) - \EXP_{t-1}[ u \cdot (-g_t)] } \leq K_2 (b+r) \del{ \sqrt{T \log\frac{1}{\delta}} + \log\frac{1}{\delta} }.
\end{equation}
The above concentration of martingales, in allusion to  Lemma~\ref{lem:lower-correlation}, gives
\begin{equation}\label{eq:tmp:u-gt-lower}
\sum_{t=1}^{T} u \cdot (-g_t) \geq \sum_{t=1}^{T} f_{u, b}(w_{t-1}) - T \cdot \sqrt{\frac{c_0c_1}{c_2}} \cdot (b+r) - K_2 (b+r) \del{ \sqrt{T \log\frac{1}{\delta}} + \log\frac{1}{\delta} }.
\end{equation}

Combining \eqref{eq:tmp:omd-regret}, \eqref{eq:tmp:breg}, \eqref{eq:tmp:g_t}, and \eqref{eq:tmp:u-gt-lower}, we obtain
\begin{align*}
C \cdot \frac{1}{T} \sum_{t=1}^{T} f_{u, b}(w_{t-1}) \leq&\  (b+r) \del{ \frac{\sqrt{\log(1/\delta')}}{\sqrt{T}} + \frac{\log(1/\delta')}{T} + C \sqrt{\frac{c_0 c_1}{c_2}}} \\
&\ +  \frac{\rho^2 \log d}{ \alpha T} + \alpha  \cdot  \log^2\frac{Td}{b \delta}
\end{align*}
for some absolute constant $C > 0$.
\end{proof}

The following proposition is an immediate result of Lemma~\ref{lem:omd} by specifying the involved hyper-parameters and showing that $u$ stays in the convex constraint set $\calK$.

\begin{proposition}\label{prop:avg-f-refine}
Suppose that the adversarial noise rate $\nu \leq c_0 \epsilon$ for some sufficiently small absolute constant $c_0 > 0$. Consider running the \refine algorithm with step size $\alpha = \tilde{\Theta}\del{ \theta \cdot \log^{-2}\frac{d}{\delta  \theta}}$, bandwidth $b = \Theta(\theta)$, convex constraint set $\calK = \{w \in \Rd: \twonorm{w - w_0} \leq \theta, \twonorm{w} \leq 1\}$, regularizer $\Phi(w) = \frac{1}{2(p-1)}\pnorm{w - w_0}^2$, number of iterations $T = \tilde{O}(s \log d \cdot \log^2\frac{d}{\delta \theta})$. If the initial iterate $w_0$ is such that $\twonorm{w_0}=1$, $\zeronorm{w_0} \leq s$, and $\twonorm{w_0 - u} \leq  \theta$ for some $\theta \leq \frac{\pi}{16}$, then with probability $1-\delta$,
\begin{equation*}
\frac{1}{T} \sum_{t=1}^{T} f_{u, b}(w_{t-1}) \leq \frac{\theta}{50 \cdot 3^4 \cdot 2^{33}}.
\end{equation*}
\end{proposition}
\begin{proof}
We will first verify that the premises of Lemma~\ref{lem:omd} are satisfied. First of all, it is easy to see that $\onenorm{w_0 - u} \leq \sqrt{\zeronorm{w_0 - u}} \cdot \twonorm{w_0 - u} \leq \sqrt{2s} \cdot \theta$. Hence we can choose $\rho = \sqrt{2s} \cdot \theta$ in Lemma~\ref{lem:omd}. Next, we have $\twonorm{u - w_0} \leq \theta$ in view of our condition on $w_0$, which together with the fact that $\twonorm{u} =1$ implies $u \in \calK$. Last, for all $w \in \calK$, by triangle inequality $\twonorm{w - u} \leq \twonorm{w - w_0} + \twonorm{w_0 - u} \leq 2 \theta$. Hence we can choose $r = 2 \theta$ in Lemma~\ref{lem:omd}.

Now with $\rho = \sqrt{2s} \cdot \theta$ and $r = 2 \theta$, Lemma~\ref{lem:omd} indicates that with probability $1- \delta$,
\begin{equation*}
C \cdot \frac{1}{T} \sum_{t=1}^{T} f_{u, b}(w_{t-1}) \leq  (b+ 2\theta) \del{ \frac{\sqrt{\log(1/\delta)}}{\sqrt{T}} + \frac{\log(1/\delta)}{T} + C \sqrt{\frac{c_0 c_1}{c_2}} } +  \frac{\theta^2 \cdot 2s \log d}{ \alpha  T} + {\alpha}  \cdot  \log^2\frac{Td}{b \delta}.
\end{equation*}
We need each term on the right-hand side is upper bounded by $\frac{C \cdot \theta}{3 \cdot 50 \cdot 3^4 \cdot 2^{33}}$. First, we choose $b = \Theta(\theta)$. Then for the first term, it suffices to choose $T \geq \log(1/\delta)$ and set $c_0$ to be a sufficiently small absolute constant (this is possible since $C$, $c_1$, and $c_2$ are all fixed absolute constants). The second and the last terms require $\alpha T \geq \Omega(\theta \cdot s \log d)$ and $\alpha  \cdot  \log^2\frac{Td}{b \delta} = O(\theta)$ respectively. The latter implies $\alpha = O(\theta \cdot \log^{-2}\frac{Td}{b \delta})$, thus combining it with the former we need
\begin{equation*}
\frac{\theta \cdot s \log d}{T} \leq \theta \cdot \log^{-2}\frac{Td}{b \delta}.
\end{equation*}
This can be satisfied if we choose $T = \tilde{O}(s \log d \cdot \log^2\frac{d}{\delta \theta})$. Finally, we have $\alpha = \tilde{\Theta}(\theta \cdot \log^{-2} \frac{d}{\delta\theta})$.
\end{proof}

\begin{remark}
In the above proof, we note that $C \sqrt{\frac{c_0 c_1}{c_2}}$ is an extra term introduced by the adversarial noise model. It is important to observe that $c_0$ is chosen as a very small {\em absolute constant}. Thus the noise tolerance still reads as $\nu = \Omega(\epsilon)$. The term $C \sqrt{\frac{c_0 c_1}{c_2}}$ does not appear in the bounded noise analysis though; see Lemma~8 of \citet{zhang2020efficient}.
\end{remark}

\begin{lemma}\label{lem:online-to-batch-refine}
Let $\theta \in [0, \frac{\pi}{16}]$ be a given scalar. Let $w_0, \dots w_{T-1}$ be a sequence of vectors such that for all $1 \leq t \leq T$, $\twonorm{w_{t-1} - u} \leq 2 \theta$ and $\twonorm{w_{t-1}} \leq 1$. Further assume that $\frac{1}{T} \sum_{t=1}^{T} f_{u, b}(w_{t-1}) \leq \frac{\theta}{50 \cdot 3^4 \cdot 2^{33}}$. Let $\bar{w} = \frac{1}{T}\sum_{t=1}^T w_{t-1}$ and $v = \frac{\calH_s(\bar{w})}{\twonorm{\calH_s(\bar{w})}}$. Then $\theta(v, u) \leq \frac{\theta}{2}$.
\end{lemma}
\begin{proof}
Define the index set $S = \cbr{t \in [T]: f_{u, b}(w_{t-1}) \geq \frac{\theta}{5 \cdot 3^4 \cdot 2^{21}} }$. It is easy to show that $\frac{\abs{S}}{T} \leq \frac{1}{10 \cdot 2^{12}}$ as otherwise the average of $f_{u, b}(w_{t-1})$ will exceed the assumed upper bound. Therefore, $\frac{\abs{\bar{S}}}{T} \geq 1 - \frac{1}{10 \cdot 2^{12}}$. For all $t \in \bar{S}$ we have $f_{u, b}(w_{t-1}) \leq \frac{\theta}{50 \cdot 3^4 \cdot 2^{21}}$; by Lemma~\ref{lem:f-to-angle}, we have $\theta(w_{t-1}, u) \leq \frac{\theta}{5}$ for these $t$.

Now consider $\theta(w_{t-1}, u)$ for $t \in S$. As we showed in the proof of Proposition~\ref{prop:avg-f-refine}, we have $\twonorm{w_{t-1} - u } \leq 2 \theta$. Since $\twonorm{u} = 1$ and $\twonorm{w_{t-1}} \leq 1$, we use the basic fact  that $\theta(w_{t-1}, u) \leq \pi \twonorm{w_{t-1} - u} < 8 \theta$.

Now we translate these bounds on the angles to those of the cosine distance, and obtain
\begin{align*}
\frac1T \sum_{t=1}^T \cos\theta(w_{t-1}, u)
&\geq \cos\frac{\theta}{5} \cdot \del{1 - \frac{1}{20 \cdot 2^{12}}} + \cos(8\theta) \cdot \frac{1}{20 \cdot 2^{12}} \\
&\geq \del{ 1 - \frac{\theta^2}{50} } \del{1 - \frac{1}{20 \cdot 2^{12}}} + \del{1 - \frac{(8\theta)^2}{2}} \frac{1}{20 \cdot 2^{12}} \\
&\geq   1 - \frac{1}{5} \del{\frac{\theta}{32}}^2  \geq  \cos\frac{\theta}{32}.
 \end{align*}
where in the second inequality we use the fact $\cos \theta \geq 1 - \frac{\theta^2}{2}$ for any $\theta \in [0, \pi]$, and in the last inequality we use the fact that $\cos \theta \leq 1 - \frac{\theta^2}{5}$.

The above inequality, in combination with Lemma~\ref{lem:avg-angle} yields the following guarantee for $\bar{w} = \frac1T \sum_{t=1}^T w_{t-1}$:
\begin{equation*}
\cos \theta(\bar{w}, u) \geq \frac1T\sum_{t=1}^T \cos\theta(w_{t-1}, u) \geq \cos \frac{\theta}{32}.
\end{equation*}
Finally, we use Lemma~\ref{lem:tool} to show that $\theta(v, u) \leq \pi \twonorm{v - u} \leq 4 \pi \twonorm{\bar{w} - u} \leq 16 \cdot \theta(\bar{w}, u) \leq \frac{\theta}{2}$, which concludes the proof.
\end{proof}

\begin{theorem}[Restatement of Theorem~\ref{thm:main-refine}]\label{thm:main-refine-restate}
Suppose that the adversarial noise rate $\nu \leq c_0 \epsilon$ for some sufficiently small absolute constant $c_0 > 0$. Consider running the \refine algorithm with step size $\alpha = \tilde{\Theta}\del{ \theta \cdot \log^{-2}\frac{d}{\delta'  \theta}}$, bandwidth $b = \Theta(\theta)$, convex constraint set $\calK = \{w \in \Rd: \twonorm{w - w_0} \leq \theta, \twonorm{w} \leq 1\}$, regularizer $\Phi(w) = \frac{1}{2(p-1)}\pnorm{w - w_0}^2$, number of iterations $T = \tilde{O}(s \log d \cdot \log^2\frac{d}{\delta' \theta})$. If the initial iterate $w_0$ is such that $\twonorm{w_0}=1$, $\zeronorm{w_0} \leq s$, and $\twonorm{w_0 - u} \leq  \theta$ for some $\theta \leq \frac{\pi}{16}$, then the output of the \refine algorithm, $\tilde{w}$, satisfies $\theta(\tilde{w}, u) \leq \frac{\theta}{2}$ with probability $1- \delta'$. In addition, the label complexity of \refine is $T$, and the sample complexity is $O(T/b + T \log\frac{T}{\delta})$.
\end{theorem}
\begin{proof}
The result of $\theta(\tilde{w}, u) \leq \frac{\theta}{2}$ follows from combining Proposition~\ref{prop:avg-f-refine} and Lemma~\ref{lem:online-to-batch-refine}. In particular, all the required conditions in Proposition~\ref{prop:avg-f-refine} are assumed here. The condition $\twonorm{w_{t-1} - u}$ appearing in Lemma~\ref{lem:online-to-batch-refine} can easily be verified by observing $\twonorm{w_{t-1} - w_0} \leq \theta$ and $\twonorm{w_0 - u} \leq \theta$.

The label complexity bound is exactly $T$ since \refine runs in $T$ iterations and requests one label per iteration. Since the marginal distribution is assumed to be isotropic log-concave, Lemma~\ref{lem:logconcave} shows that $\Pr_{x_t \sim D_X}( x_t \in X_{\hatw_{t-1}, b}) \geq c_2 b$ for some absolute constant $c_2 > 0$. Thus, by Chernoff bound, we need to call $\oraclexy$ for $O(b^{-1} + \log\frac{T}{\delta})$ times in order to obtain one $x_t$ with probability $1 - \frac{\delta}{2T}$. Thus, by union bound over the $T$ iterations in \refine, with probability $1 - \frac{\delta}{2}$, the total number of calls to $\oraclexy$ is $O(T/b + T \log\frac{T}{\delta})$.
\end{proof}

\subsection{Analysis of \init}

In this subsection, we use $\EXP[\cdot]$ denote the expectation $\EXP_{(x, y) \sim D}[\cdot]$ and likewise for $\Pr(\cdot)$.
\begin{lemma}\label{lem:E[yux]}
Suppose that $\nu \leq \frac14$. Then $\EXP[y(u \cdot x)] \geq \frac{1}{9 \cdot 2^{17}}$.
\end{lemma}
\begin{proof}
We have
\begin{align*}
\EXP[ y (u \cdot x)] =&\ \EXP[ y (u \cdot x) \mid y = \sign(u\cdot x)] \cdot \Pr(y = \sign(u\cdot x)) \\
&\ + \EXP[ y (u \cdot x) \mid y \neq \sign(u\cdot x)] \cdot \Pr(y \neq \sign(u\cdot x))\\
\geq&\ \EXP[\abs{u \cdot x}] \cdot (1 - \nu) - \EXP[\abs{u \cdot x}] \cdot \nu\\
=&\ (1- 2 \nu) \EXP[\abs{u \cdot x}].
\end{align*}
Since $u \cdot x$ is an isotropic log-concave random variable in $\R$, its density function is lower bounded by $2^{-16}$ when $\abs{u \cdot x} \leq \frac19$ in view of Lemma~\ref{lem:logconcave}. Thus $\EXP[\abs{u \cdot x}] \geq \frac{1}{9 \cdot 2^{16}}$. On the other side, we assumed $\nu \leq \frac14$. Together, we obtain $\EXP[y(u \cdot x)] \geq \frac{1}{9 \cdot 2^{17}}$.
\end{proof}

\begin{lemma}\label{lem:wavg-u-cor}
Let $m = O(\log\frac{1}{\delta})$ and let $(x_1, y_1), \dots, (x_m, y_m)$ be $m$ i.i.d. samples drawn from $D$. Then with probability $1- \delta$, 
\begin{equation*}
\wavg \cdot u \geq \frac{1}{9 \cdot 2^{18}},
\end{equation*}
where $\wavg := \frac{1}{m}\sum_{i=1}^{m} y_i x_i$.
\end{lemma}
\begin{proof}
First, Lemma~\ref{lem:logconcave} shows that $u \cdot x$ is isotropic log-concave, and hence $y (u\cdot x)$ is a $(32, 16)$-subexponential random variable by Lemma~34 of \citet{zhang2020efficient}. Therefore, the standard concentration bound implies that there is an absolute constant $K_1 > 0$, such that if $m = O(\log\frac{1}{\delta})$, with probability $1-\frac{\delta}{2}$,
\begin{equation*}
\abs{ \frac{1}{m} \sum_{i=1}^{m} y_i (u \cdot x_i) - \EXP[y (u \cdot x)] } \leq K_1\del{ \sqrt{\frac{\log(1/\delta)}{m}} + \frac{\log(1/\delta)}{m} } \leq \frac{1}{9 \cdot 2^{18}}.
\end{equation*}
This in allusion to Lemma~\ref{lem:E[yux]} gives $\frac{1}{m} \sum_{i=1}^{m} y_i (u \cdot x_i) \geq \frac{1}{9 \cdot 2^{18}}$, namely 
\begin{equation*}
\wavg \cdot u \geq \frac{1}{9 \cdot 2^{18}},
\end{equation*}
which is the desired lower bound.
\end{proof}

\begin{lemma}\label{lem:wsharp-u-cor}
Let $\tilde{s} \leq d$ be a positive integer, and set $m = O(\tilde{s} \log\frac{d}{\delta})$. Let $(x_1, y_1), \dots, (x_m, y_m)$ be $m$ i.i.d. samples drawn from $D$. Then with probability $1- \delta$,
\begin{equation*}
\wsharp \cdot u \geq \frac{1}{9 \cdot 2^{20}}.
\end{equation*}
\end{lemma}
\begin{proof}
Denote $w' = \calH_{\tilde{s}}(\wavg)$. Using Lemma~17 of \citet{zhang2020efficient} we know that with probability $1-\frac{\delta}{2}$, $\twonorm{w'} \leq 2$. From the choice of $m$ and Lemma~\ref{lem:wavg-u-cor}, we have $\wavg \cdot u \geq \frac{1}{9 \cdot 2^{18}}$ with probability $1-\frac{\delta}{2}$. We hence condition on both events happening.

Now Lemma~16 of \citet{zhang2020efficient} implies that
\begin{equation*}
\abs{ w' \cdot u - \wavg \cdot u} \leq \sqrt{\frac{s}{\tilde{s}}} \twonorm{w'} \leq 2 \sqrt{\frac{s}{\tilde{s}}}.
\end{equation*}
Therefore, 
\begin{equation*}
w' \cdot u \geq \wavg \cdot u - 2 \sqrt{\frac{s}{\tilde{s}}} \geq \frac{1}{9 \cdot 2^{18}} - 2 \sqrt{\frac{s}{\tilde{s}}}.
\end{equation*}
Now taking $\tilde{s} = 81 \cdot 2^{40} s$ gives us $w' \cdot u \geq \frac{1}{9 \cdot 2^{19}}$. Finally, by algebra
\begin{equation*}
\wsharp \cdot u = \frac{1}{\twonorm{w'}} (w' \cdot u) \geq \frac{1}{2} \cdot \frac{1}{9 \cdot 2^{19}} = \frac{1}{9 \cdot 2^{20}}.
\end{equation*}
The proof is complete.
\end{proof}

\begin{proposition}\label{prop:avg-f-init}
Suppose that the adversarial noise rate $\nu \leq c_0 \epsilon$ for some sufficiently small absolute constant $c_0 > 0$. Let $\zeta = \frac{1}{9 \cdot 2^{20}}$. Consider running the \init algorithm with step size $\alpha = \tilde{\Theta}\del{ \log^{-2}\frac{d}{\delta }}$, bandwidth $b = \Theta(1)$, convex constraint set $\calK = \{w \in \Rd: \twonorm{w} \leq 1,\ \wsharp \cdot w \geq \zeta \}$, regularizer $\Phi(w) = \frac{1}{2(p-1)}\pnorm{w - w_0}^2$, number of iterations $T = \tilde{O}(s \log d \cdot \log^2\frac{d}{\delta })$, where $w_0$ is an arbitrary point in $\calK \cap \{w \in \Rd: \onenorm{w} \leq \sqrt{s}\}$ that can be found in polynomial time. Then with probability $1-\delta$,
\begin{equation*}
\frac{1}{T} \sum_{t=1}^{T} f_{u, b}(w_{t-1}) \leq \frac{\zeta}{20 \cdot 3^4 \cdot 2^{33}}.
\end{equation*}
\end{proposition}
\begin{proof}
We will first verify that the premises of Lemma~\ref{lem:omd} are satisfied. First of all, it is easy to see that $\onenorm{w_0 - u}  \leq 2\sqrt{s}$. Hence we can choose $\rho = 2\sqrt{s}$ in Lemma~\ref{lem:omd}. Next, we have $\twonorm{u} = 1$, which together with Lemma~\ref{lem:wsharp-u-cor} implies $u \in \calK$ with probability $1- \frac{\delta}{2}$. Last, for all $w \in \calK$, by triangle inequality $\twonorm{w - u} \leq \twonorm{w } + \twonorm{u} \leq 2$. Hence we can choose $r = 2$ in Lemma~\ref{lem:omd}.

Now with $\rho = \sqrt{s}$ and $r = 2$, Lemma~\ref{lem:omd} indicates that with probability $1- \frac{\delta}{2}$,
\begin{equation*}
C \cdot \frac{1}{T} \sum_{t=1}^{T} f_{u, b}(w_{t-1}) \leq  (b+ 2) \del{ \frac{\sqrt{\log(1/\delta)}}{\sqrt{T}} + \frac{\log(1/\delta)}{T} + C \sqrt{\frac{c_0 c_1}{c_2}} } +  \frac{ s \log d}{ \alpha  T} + {\alpha}  \cdot  \log^2\frac{Td}{b \delta}.
\end{equation*}
We need each term on the right-hand side is upper bounded by $\frac{C \cdot \zeta }{  20 \cdot 3^5 \cdot 2^{33}}$. First, we choose $b = \Theta(1)$. Then for the first term, it suffices to choose $T \geq \log(1/\delta)$ and set $c_0$ to be a sufficiently small absolute constant (this is possible since $C$, $c_1$, and $c_2$ are all fixed absolute constants). The second and the last terms require $\alpha T \geq \Omega(s \log d)$ and $\alpha  \cdot  \log^2\frac{Td}{b \delta} = O(1)$ respectively. The latter implies $\alpha = O(\log^{-2}\frac{Td}{b \delta})$, thus combining it with the former we need
\begin{equation*}
\frac{ s \log d}{T} \leq  \cdot \log^{-2}\frac{Td}{b \delta}.
\end{equation*}
This can be satisfied if we choose $T = \tilde{O}(s \log d \cdot \log^2\frac{d}{\delta})$, which results in $\alpha = \tilde{\Theta}( \log^{-2} \frac{d}{\delta})$.
\end{proof}

\begin{lemma}\label{lem:online-to-batch-init}
Set $b = \frac{1}{81 \cdot 2^{22}}$ and let $\zeta = \frac{1}{9 \cdot 2^{20}}$. Let $w_0, \dots w_{T-1}$ be a sequence of vectors such that for all $1 \leq t \leq T$, $\twonorm{w_{t-1}} \leq 1$. Further assume that $\frac{1}{T} \sum_{t=1}^{T} f_{u, b}(w_{t-1}) \leq \frac{\zeta}{20 \cdot 3^4 \cdot 2^{33}}$. Let $\bar{w} = \frac{1}{T}\sum_{t=1}^T w_{t-1}$ and $v = \frac{\calH_s(\bar{w})}{\twonorm{\calH_s(\bar{w})}}$. Then $\theta(v, u) \leq \frac{\pi}{8}$.
\end{lemma}
\begin{proof}
In light of Lemma~\ref{lem:wsharp-u-cor}, we know that $u \in \calK$. Also, our choices of $b$ and $\zeta$ implies that $b \leq \frac{\zeta}{36}$. Define $S = \{1 \leq t \leq T: f_{u, b}(w_{t-1}) \geq \frac{\zeta}{3^4 \cdot 2^{21}} \}$. Then by the second part of Lemma~\ref{lem:f-to-angle}, for all $t \in \bar{S}$, we have $\theta(w_{t-1}, u) < \zeta$. For all $t \in S$, we have a trivial estimate of $\theta(w_{t-1}, u) \in [0, \pi]$.

Next we bound the size of $S$. Using the condition $\frac{1}{T} \sum_{t=1}^{T} f_{u, b}(w_{t-1}) \leq \frac{\zeta}{20 \cdot 3^4 \cdot 2^{33}}$, it is possible to show that $\frac{\abs{S}}{T} \leq \frac{1}{20 \cdot 2^{12}}$ and thus $\frac{\abs{\bar{S}}}{T} \geq 1 - \frac{1}{20 \cdot 2^{12}}$.

Now we translate these bounds on the angles to those of the cosine distance, and obtain
\begin{align*}
\frac1T \sum_{t=1}^T \cos\theta(w_{t-1}, u)
&\geq \cos\zeta \cdot \del{1 - \frac{1}{20 \cdot 2^{12}}} + (\cos \pi ) \cdot \frac{1}{20 \cdot 2^{12}} \\
&\geq   1 - \frac{1}{5} \del{\frac{\pi}{128}}^2  \geq  \cos\frac{\pi}{128}.
 \end{align*}
where in the last inequality we use the fact that $\cos \theta \leq 1 - \frac{\theta^2}{5}$.

The above inequality, in combination with Lemma~\ref{lem:avg-angle} yields the following guarantee for $\bar{w} = \frac1T \sum_{t=1}^T w_{t-1}$:
\begin{equation*}
\cos \theta(\bar{w}, u) \geq \frac1T\sum_{t=1}^T \cos\theta(w_{t-1}, u) \geq \cos \frac{\pi}{128}.
\end{equation*}
Finally, we use Lemma~\ref{lem:tool} to show that $\theta(v, u) \leq \pi \twonorm{v - u} \leq 4 \pi \twonorm{\bar{w} - u} \leq 16 \cdot \theta(\bar{w}, u) \leq \frac{\pi}{8}$, which concludes the proof.
\end{proof}

\begin{theorem}[Restatement of Theorem~\ref{thm:main-init}]\label{thm:main-init-restate}
Suppose that the adversarial noise rate $\nu \leq c_0 \epsilon$ for some sufficiently small absolute constant $c_0 > 0$. Let $\zeta = \frac{1}{9 \cdot 2^{20}}$. Consider running the \init algorithm with step size $\alpha = \tilde{\Theta}\del{ \log^{-2}\frac{d}{\delta }}$, bandwidth $b = \Theta(1)$, convex constraint set $\calK = \{w \in \Rd: \twonorm{w} \leq 1,\ \wsharp \cdot w \geq \zeta \}$, regularizer $\Phi(w) = \frac{1}{2(p-1)}\pnorm{w - w_0}^2$, number of iterations $T = \tilde{O}(s \log d \cdot \log^2\frac{d}{\delta })$, where $w_0$ is an arbitrary point in $\calK \cap \{w \in \Rd: \onenorm{w} \leq \sqrt{s}\}$ that can be found in polynomial time. Then with probability $1-\delta$, the output of \init, $v_0$, is such that $\theta(v_0, u) \leq \frac{\pi}{8}$.
\end{theorem}
\begin{proof}
This is an immediate result by combining Proposition~\ref{prop:avg-f-init} and Lemma~\ref{lem:online-to-batch-init}.
\end{proof}

\section{The Structure of $f_{u, b}(w)$}

Our definition of the potential function $f_{u, b}(w)$ slightly differs from that of \citet{zhang2020efficient}: in this work, the expectation is taken over $D$ conditioned on $\{x \in \Rd: 0 < w \cdot x \leq b\}$ while in \citet{zhang2020efficient} it is conditioned on $\{x \in \Rd: -b \leq w \cdot x \leq b\}$. It can be seen that our function value is always less than that of \citet{zhang2020efficient}. However, we note that the difference in sampling region does not lead to significant difference in the structure of the potential function. In particular, we are still able to show that under certain conditions, $f_{u, b}(w)$ serves as an upper bound of $\theta(w, u)$~--~a crucial observation made in \citet{zhang2020efficient}.

\begin{lemma}\label{lem:f>angle}
Let $w$ and $u$ be two unit vectors. Suppose $b \in \big[0, \frac{\pi}{72} \big]$. We have
\begin{enumerate}
\item If $\theta(w, u) \in [36b, \frac{\pi}{2}]$, then $f_{u, b}(w) \geq \frac{\theta(w, u)}{3^4 \cdot 2^{21}}$.

\item If $\theta(w, u) \in [\frac{\pi}{2}, \pi - 36b]$, then $f_{u, b}(w) \geq \frac{\pi - \theta(w, u)}{3^4 \cdot 2^{21}}$.
\end{enumerate}
\end{lemma}
\begin{proof}
The proof of the first part follows closely from [Lemma~22, Part~1] of \citet{zhang2020efficient}. In particular, for the region
\begin{equation*}
R_1 := \cbr{x \in \Rd: w \cdot x \in [0, b],\ u \cdot x \in \sbr{- \frac{\sin \theta(w, u)}{18}, -\frac{\sin \theta(w, u)}{36}}},
\end{equation*}
their analysis shows that
\begin{equation}
\Pr_{x \sim D_X}( x \in R_1) \geq \frac{b}{9 \cdot 2^{18}}.
\end{equation}
To see why it completes the proof of the first part, observe that
\begin{align*}
 \EXP_{x \sim D_X} \sbr{\abs{u \cdot x} \cdot \ind{{ 0 \leq w \cdot x \leq b}} \cdot \ind{u \cdot x < 0}} \geq&\ \EXP_{x \sim D_X} \sbr{\abs{u \cdot x} \cdot \ind{x \in R_1}} \\
\geq&\ \frac{\sin\theta(w, u)}{36} \cdot \EXP_{x \sim D_X} \sbr{\ind{x \in R_1}} \\
\geq&\ \frac{\theta(w, u)}{72} \cdot \Pr_{x \sim D_X}\del{x \in R_1}
\geq \frac{\theta(w, u) \cdot b}{3^4 \cdot 2^{21}},
\end{align*}
where the first inequality uses the fact that $R_1$ is a subset of both sets $\cbr{x \in \Rd:  w \cdot x \in [0, b]}$ and $\cbr{x \in \Rd: u \cdot x < 0}$; the second inequality uses the fact that for all $x$ in $R_1$, $\abs{u \cdot x} \geq \frac{\sin \theta(w, u)}{36}$; the third inequality uses the elementary fact that $\sin\phi \geq \frac{\phi}{2}$ for any angle $\phi \in [0, \frac{\pi}{2}]$.

As $\Pr_{x \sim D_X} \del{ w \cdot x \in [0,  b] } \leq b$ by Lemma~\ref{lem:logconcave}, we have
\begin{equation*}
f_{u,b}(w) = \frac{\EXP_{x \sim D_X} \sbr{\abs{u \cdot x} \cdot \ind{{ 0 \leq w \cdot x \leq b}} \cdot \ind{u \cdot x < 0}}}{\Pr_{x \sim D_X} \del{ w \cdot x \in [0,  b] } }
\geq \frac{\theta(w, u) \cdot b}{3^4 \cdot 2^{21}} \cdot \frac{1}{b}= \frac{\theta(w, u)}{3^4 \cdot 2^{21}}.
\end{equation*}
This completes the proof of the first part.

For the second part, we will define $\phi = \pi - \theta(w, u)$ and consider
\begin{equation*}
R_2 := \cbr{x \in \Rd: w \cdot x \in [0, b],\ u \cdot x \in \sbr{- \frac{\sin \phi}{18}, -\frac{\sin \phi}{36}}}.
\end{equation*}
Similar to the region $R_1$, we have $\Pr_{x \sim D_X}(x \in R_2) \geq \frac{b}{9 \cdot 2^{18}}$. Hence, using the same induction with the proof of first part, we have $f_{u, b}(w) \geq \frac{\phi}{3^4 \cdot 2^{21}}= \frac{\pi - \theta(w, u)}{3^4 \cdot 2^{21}}$.
\end{proof}

The following lemma connects the potential function $f_{u, b}(w)$ to the angle $\theta(w, u)$.

\begin{lemma}\label{lem:f-to-angle}
Let $w$ and $u$ be two unit vectors. We have the following:
\begin{enumerate}
\item Suppose $\theta \in [0, \frac{\pi}{2}]$ is given. Set $b \leq \frac{\theta}{5 \cdot 36}$. If $f_{u,b}(w) \leq \frac{\theta}{5 \cdot 3^4 \cdot 2^{21}}$, then $\theta(w,u) \leq \frac{\theta}{5}$.

\item Let $\wsharp \in \Rd$ be a unit vector with $\wsharp \cdot u \geq \zeta$ for some $\zeta \in (0, 1)$. Set $b \leq \frac{\zeta}{36}$. If $w$ is in $\calK := \{w \in \Rd: \twonorm{w} \leq 1,\ w \cdot \wsharp \geq \zeta\}$ and $f_{u,b}(w) < \frac{\zeta}{3^4 \cdot 2^{21}}$, then $\theta(w, u) < \zeta$.
\end{enumerate}
\end{lemma}

\begin{proof}
The first part was already set out in Claim~10 of \citet{zhang2020efficient}. For the second part, first, we show that it is impossible for $\theta(w, u) \geq \frac{\pi}{2}$. Assume for contradiction that this holds.
By Lemma~\ref{lem:angle-not-flat}, for all $w$ in $\calK$, we have that
$\theta(w, u) \leq \pi - \zeta$. By the choice of $b$, we know that $36b \leq \zeta$, hence $\theta(w, u) \leq \pi - 36 b$. Now using the second part of Lemma~\ref{lem:f>angle}, we have
\begin{equation}
f_{u,b}(w) \geq \frac{\pi - \theta(w, u)}{3^4 \cdot 2^{21}} \geq \frac{\zeta}{3^4 \cdot 2^{21}},
\end{equation}
which contradicts with the premise that $f_{u,b}(w) < \frac{\zeta}{3^4 \cdot 2^{21}}$.

Therefore, $\theta(w, u) \in [0, \frac \pi 2]$. We now conduct a case analysis. If $\theta(w, u) \leq 36 b$, then by the definition of $b$, we automatically have $\theta(w, u) < \zeta$. Otherwise, $\theta(w, u) \in [36b, \frac \pi 2]$. In this case, the first part of Lemma~\ref{lem:f>angle} implies
\begin{equation*}
f_{u,b}(w) \geq \frac{\theta(w, u)}{3^4 \cdot 2^{21}}.
\end{equation*}
This inequality, in conjunction with the condition that $f_{u,b}(w) < \frac{\zeta}{3^4 \cdot 2^{21}}$, implies that $\theta(w, u) \leq \zeta$. In summary, in both cases, we have $\theta(w, u) \leq \zeta$. This completes the proof.
\end{proof}

\begin{lemma}[Lemma~19 of \citet{zhang2020efficient}]\label{lem:angle-not-flat}
Let $\wsharp \in \Rd$ be a unit vector, and $\zeta \in (0, 1)$. For any two vectors $w$ and $v$ in the set $\calK = \{w \in \Rd: \twonorm{w} \leq 1,\ w \cdot \wsharp \geq \zeta\}$, it holds that $\theta(w, v) \leq \pi - \zeta$.
\end{lemma}

\section{Useful Lemmas}

We collect a few useful results that are frequently invoked in our analysis.

\begin{lemma}[Lemma 3.4 of \citet{awasthi2017power}]\label{lem:E[wx^2]}
There is an absolute constant $c_1 > 0$ such that the following holds. Let $D_X$ be an isotropic log-concave distribution. Fix $w \in \Rd$. For all $v$ with $\twonorm{v - w} \leq r$, $\EXP_{x \sim D_{X | \hatw, b}}[ (v \cdot x)^2 ] \leq c_1 (b^2 + r^2)$.
\end{lemma}

\begin{lemma}[\citet{lovasz2007geometry}]\label{lem:logconcave}
There exists an absolute constants $c_2, c_3 > 0$ such that the following holds. Let $D_X$ be an isotropic log-concave distribution over $\Rd$.
\begin{enumerate}
\item Given any unit vector $w$, $w \cdot x$ is isotropic log-concave if $x$ is drawn from $D_X$.

\item Given any unit vector $w \in \Rd$, $c_2 b \leq \Pr_{x \sim D_X}( {w \cdot x} \in [0, b] ) \leq  b$.

\item If $d=1$ or $d=2$, for all $x \in \Rd$ with $\twonorm{x} \leq \frac19$, the density function $p(x) \geq 2^{-16}$.
\end{enumerate}
\end{lemma}

\begin{lemma}[Lemma~16 of \citet{shen2020attribute}]\label{lem:|x|_inf-D_ub}
There exists an absolute constant $c_3 > 0$ such that the following holds for all isotropic log-concave distributions $D_X$. Let $S$ be a set of i.i.d. instances drawn from $D_{X| \hatw, b}$. Then
\begin{equation*}
\Pr_{S \sim D_{X|\hatw, b}^n}\del{ \max_{x \in S} \infnorm{x} \geq c_3 \log\frac{\abs{S}d}{ b\delta} } \leq \delta.
\end{equation*}
\end{lemma}

\begin{lemma}[Lemma~36 of \citet{zhang2020efficient}]\label{lem:subexp-tail-hoeff}
Suppose $\cbr{Z_t}_{t=1}^T$ is sequence of random variables adapted to filtration $\cbr{\calF_t}_{t=1}^T$. Denote by $\Pr_{t-1}( \cdot )$ and $\EXP_{t-1}[\cdot]$ the probability and expectation conditioned on $\calF_{t-1}$, respectively.
For every $Z_t$, suppose that $\Pr_{t-1}( \abs{Z_t} > a ) \leq C \exp\del{-\frac{a}{\sigma}}$ for some absolute constant $C \geq 1$.
Then, with probability $1-\delta$,
\begin{equation*}
\abs{\sum_{t=1}^T Z_t - \EXP_{t-1} [Z_t] } \leq 16\sigma(\ln C + 1) \del{\sqrt{2 T \ln\frac{2}{\delta}} + \ln\frac{2}{\delta}}.
\end{equation*}
\end{lemma}

\begin{lemma}[Lemma 24 of \citet{zhang2020efficient}]\label{lem:avg-angle}
Suppose we have a sequence of unit vectors $w_0, \dots, w_{T-1}$. Let $\bar{w} = \frac1T \sum_{t=1}^T w_{t-1}$ be their average.
Suppose $\frac1T \sum_{t=1}^T \cos\theta(w_{t-1}, u) \geq 0$.
Then,
$\cos\theta(\bar{w}, u) \geq \frac1T \sum_{t=1}^T \cos\theta(w_{t-1}, u)$.
\end{lemma}

\begin{lemma}\label{lem:tool}
Let $w$ and $v$ be two vectors in $\Rd$. The following holds:
\begin{enumerate}
\item If $v$ is a unit vector, then $\twonorm{\hatw - v} \leq 2 \twonorm{w - v}$.
\item If $v$ is $s$-sparse, then $\twonorm{\calH_s(w) -  v} \leq 2 \twonorm{w - v}$.
\item If $v$ is a unit vector, then $\theta(w, v) \leq \pi \twonorm{w - v}$; if $w$ is a unit vector as well, then we further have $\twonorm{w - v} \leq \theta(w, v)$.
\end{enumerate}
\begin{proof}
The first two parts are known expansion error of $\ell_2$-normalization and hard thresholding, respectively. The proof can be found in, e.g. \citet{shen2018tight} (which also presents a sharp bound for the second part). The last part can be derived by fundamental algebra.
\end{proof}
\end{lemma}